\documentclass[11pt]{article}

\usepackage[T1]{fontenc}
\usepackage[utf8]{inputenc}
\usepackage[a4paper,margin=1in]{geometry}
\usepackage{microtype}
\usepackage{setspace}
\usepackage{titlesec}
\usepackage{authblk}

\usepackage[english]{babel}
\usepackage{csquotes}

\usepackage{amsmath,amssymb,amsfonts,mathtools,mathrsfs}
\usepackage{bm}
\usepackage{siunitx}
\usepackage{physics}
\usepackage{enumitem}

\DeclareMathOperator{\conv}{conv}

\newcommand{\R}{\mathbb{R}}

\newcommand{\1}{\mathbf{1}}
\newcommand{\Newt}{\mathrm{Newt}}        
\newcommand{\Trop}{\mathcal{T}}          
\newcommand{\Argmax}{\operatorname{Argmax}}

\usepackage{amsthm}
\numberwithin{equation}{section}
\newtheorem{theorem}{Theorem}[section]
\newtheorem{proposition}[theorem]{Proposition}
\newtheorem{lemma}[theorem]{Lemma}

\theoremstyle{definition}
\newtheorem{definition}[theorem]{Definition}

\newtheorem{example}[theorem]{Example}
\theoremstyle{remark}
\newtheorem{remark}[theorem]{Remark}

\usepackage{graphicx}
\usepackage{xcolor}
\usepackage{booktabs,multirow,makecell}
\usepackage{caption}
\usepackage{subcaption}
\usepackage{float}
\usepackage{wrapfig}
\usepackage{placeins}
\usepackage{tikz}
\usetikzlibrary{calc,arrows.meta,positioning}
\usepackage{tikz-3dplot}
\usepackage{pgfplots}
\pgfplotsset{compat=1.18}

\usepackage{algorithm}
\usepackage{algpseudocode}
\algrenewcommand\algorithmicrequire{\textbf{Input:}}
\algrenewcommand\algorithmicensure{\textbf{Output:}}

\usepackage{listings}
\usepackage[numbers,sort&compress]{natbib} 
\usepackage[hidelinks]{hyperref}
\hypersetup{
  colorlinks=true,
  linkcolor=blue!60!black,
  citecolor=blue!60!black,
  urlcolor=blue!60!black,
  pdftitle={From Universal Approximation Theorem to Tropical Geometry of Multi-Layer Perceptrons}
}
\usepackage[nameinlink,noabbrev]{cleveref}

\title{From Universal Approximation Theorem to Tropical Geometry of Multi-Layer Perceptrons}

\author[1]{Yi-Shan Chu\thanks{Corresponding author. Email: \href{mailto:easonchu7@gmail.com}{easonchu7@gmail.com}}}
\author[1]{Yueh-Cheng Kuo\thanks{Email: \href{mailto:kuoyc@nccu.edu.tw}{kuoyc@nccu.edu.tw}}}

\affil[1]{Department of Mathematical Sciences, National Chengchi University, Taipei, Taiwan}

\date{\today}

\begin{document}
\maketitle

\begin{abstract}
We revisit the Universal Approximation Theorem(UAT) through the lens of the tropical geometry of neural networks and introduce a constructive, geometry-aware initialization for sigmoidal multi-layer perceptrons (MLPs). Tropical geometry shows that Rectified Linear Unit (ReLU) networks admit decision functions with a combinatorial structure often described as a tropical rational, namely a difference of tropical polynomials. Focusing on planar binary classification, we design purely sigmoidal MLPs that adhere to the finite-sum format of UAT: a finite linear combination of shifted and scaled sigmoids of affine functions. The resulting models yield decision boundaries that already align with prescribed shapes at initialization and can be refined by standard training if desired. This provides a practical bridge between the tropical perspective and smooth MLPs, enabling interpretable, shape-driven initialization without resorting to ReLU architectures. We focus on the construction and empirical demonstrations in two dimensions; theoretical analysis and higher-dimensional extensions are left for future work.
\end{abstract}


\section{Introduction}

Multi-layer perceptrons (MLPs) are a foundational model class with numerous extensions and architectural variants, and they have demonstrated empirical success across artificial intelligence, computer vision, robotics, and beyond. Despite this breadth of applications, our theoretical understanding of MLPs and the connections among different viewpoints remain incomplete.

A classical line of work establishes \emph{existence} results: under mild assumptions on the activation function and with suitable choices of parameters, finite linear combinations of sigmoids of affine forms are dense in natural function spaces. In particular, Cybenko \citep{Cybenko1989} showed that, for a sigmoidal activation~$\sigma$, functions of the form
\begin{equation}
\label{eq:cybenko}
G(x)\;=\;\sum_{j=1}^{N}\alpha_{j}\,\sigma\!\big(w_{j}^{\top}x+\theta_{j}\big), \qquad x\in\mathbb{R}^{d},
\end{equation}
can approximate any continuous mapping on a compact domain to arbitrary accuracy; subsequent extensions and refinements were provided, e.g., by \citet{Hornik1991}. While these universal approximation theorems (UATs) certify \emph{what} can be represented, they do not prescribe \emph{how} to construct such networks for a desired decision set, which limits interpretability and hinders geometry-aware design.

A complementary, more combinatorial perspective comes from the tropical geometry of neural networks. For ReLU models, decision functions admit piecewise-linear descriptions with rich polyhedral structure, and the decision boundary can often be analyzed through a \emph{tropical rational} (a tropical quotient of two tropical polynomials) \citep{ZhangNaitzatLim2018,AlfarraEtAl2022,BrandenburgLohoMontufar2024}. This viewpoint highlights explicit boundary programming and dual polyhedral objects (e.g., zonotopes), yet most existing constructions are tailored to ReLU piecewise-linear networks.

In this work, we revisit UAT \emph{through} the tropical lens and develop a constructive, geometry-aware initialization for \emph{sigmoidal} MLPs that retains the finite-sum structure in \eqref{eq:cybenko}. Our approach compiles a geometric covering of a target region into the weights of a purely sigmoidal network, yielding decision boundaries that match the prescribed shape at initialization; subsequent training is optional and serves primarily as refinement. In doing so, we provide a practical bridge from tropical-style boundary reasoning to smooth MLPs within the classical UAT format.
\paragraph{Related work.}
The universal approximation literature spans several complementary directions. Beyond the earliest sigmoidal results of \citet{Cybenko1989} and extensions by \citet{Hornik1991}, there are alternative proofs and activation assumptions \citep{Funahashi1989,Leshno1993}, as well as rate and complexity analyses (e.g., Barron-type bounds) \citep{Barron1993,Pinkus1999}. Depth further changes the picture via expressivity and separation phenomena \citep{MontufarEtAl2014,Telgarsky2016}, while practical constructive schemes include radial-basis and extreme-learning approaches that select or randomize hidden units and then fit linear output layers \citep{ParkSandberg1991,HuangZhuSiew2006}. 

A complementary line analyzes piecewise-linear networks through tropical geometry, where ReLU decision functions inherit polyhedral structure and can be described via \emph{tropical rational} (differences of tropical polynomials) with dual objects such as zonotopes \citep{ZhangNaitzatLim2018,AlfarraEtAl2022}. Recent work investigates parameter spaces and ``real tropical'' viewpoints that connect continuous families of networks with combinatorial models \citep{BrandenburgLohoMontufar2024}, and explores compression and structure preservation in tropical terms \citep{FotopoulosMaragosMisiakos2024}. 

Our contribution differs in two aspects: (i) we stay entirely within the sigmoidal, finite-sum format of UAT yet construct \emph{geometrically aligned} hidden units with appropriate weights by constructing a finite ball cover of the target region; (ii) we apply the tropical perspective only as a guiding analogy for boundary programming, without relying on ReLU or log-sum-exp mechanisms. This positions our method between constructive approximation and tropical analyses, emphasizing interpretable, shape-driven initialization.

Section~\ref{sec:uat} recalls the Universal Approximation Theorem (UAT) for sigmoidal multi-layer perceptrons and summarizes the guarantees relevant to our finite-sum construction. 
Section~\ref{sec:tropical} introduces tropical algebra and the tropical-geometric viewpoint on neural decision functions, providing the terminology used throughout.
Section~\ref{sec:oned} presents a one-dimensional warm-up (\(\R \to \{0,1\}\)) that illustrates the basic cases and how MLPs act.
Section~\ref{sec:planar-convex} develops the planar case with convex target regions and unions of convex sets.
Section~\ref{sec:planar-general} consider general planar regions and apply an algorithm that compiles a ball cover into the weights of a sigmoidal MLP, together with implementation details and complexity considerations.
Section~\ref{sec:higher-d} outlines extensions beyond the plane and discusses how the construction adapts in higher dimensions.
Section~\ref{sec:apps} demonstrates applications and case studies using real data.
Section~\ref{sec:conclusion} concludes with limitations and directions for future work.

\section{Universal Approximation for Sigmoidal MLPs}\label{sec:uat}
We adopt the classical finite-sum model for single-hidden-layer networks:
\begin{equation}\label{eq:finite-sum}
  G_N(x)\;=\;\sum_{j=1}^{N}\alpha_j\,\sigma\!\big(w_j^{\top}x+\theta_j\big), 
  \qquad x\in\R^{d},
\end{equation}
where $N\in\mathbb{N}$, $\alpha_j\in\R$, $w_j\in\R^{d}$, and $\theta_j\in\R$ are free parameters, and $\sigma:\R\to\R$ is a fixed activation.

\begin{definition}[Sigmoidal activation; {\citep[Def.~1]{Cybenko1989}}]\label{def:sigmoidal}
A function $\sigma:\R\to\R$ is called \emph{sigmoidal} if it is bounded and
\[
\lim_{t\to+\infty}\sigma(t)=1,
\qquad
\lim_{t\to-\infty}\sigma(t)=0.
\]
Typical examples include the logistic function(sigmoid function) $\sigma(t)=(1+e^{-t})^{-1}$ and smooth variants such as $\tanh$ (after affine rescaling).
\end{definition}

\paragraph{Universal approximation (UAT).}
Cybenko proved that the class $\{G_N\}_{N\ge 1}$ with a sigmoidal $\sigma$ is dense in $C(K)$, the continuous real-valued functions on a compact set $K\subset\R^d$, under the uniform norm.

\begin{theorem}[Cybenko's UAT; {\citep[Thm.~2]{Cybenko1989}}]\label{thm:cybenko}
Let $\sigma$ be sigmoidal in the sense of \text{Definition} \ref{def:sigmoidal}. Then, for any compact $K\subset\R^d$, the set of finite sums \eqref{eq:finite-sum} is dense in $C(K)$ with respect to the sup norm, i.e., for every $f\in C(K)$ and every $\varepsilon>0$, there exist $N$ and parameters $\{\alpha_j,w_j,\theta_j\}_{j=1}^N$ such that $\sup_{x\in K}\lvert f(x)-G_N(x)\rvert<\varepsilon$.
\end{theorem}
We adopt the classical single-hidden-layer, finite-sum model defined as in \eqref{eq:finite-sum}.
Subsequent work broadened both the function classes and the activation assumptions. Funahashi established a closely related universality result for continuous mappings on compact sets using analytic arguments \citep{Funahashi1989}. Hornik extended universality beyond uniform approximation on $C(K)$ to vector-valued outputs and to $L^p(\mu)$ spaces, and relaxed activation requirements far beyond the sigmoidal case \citep{HornikStinchcombeWhite1989,Hornik1991}. A sharp activation criterion was later given by Leshno, Lin, Pinkus, and Schocken, who characterized universality purely in terms of $\sigma$ \citep{Leshno1993}.

\begin{theorem}[Nonpolynomial criterion {\citep[Thm.~3.2]{Leshno1993}}]\label{thm:nonpoly}
Let $K\subset\R^d$ be compact with nonempty interior. The linear span of $\{\,\sigma(w^\top x+\theta):w\in\R^d,\theta\in\R\,\}$ is dense in $C(K)$ if and only if $\sigma$ is a. e.  not a polynomial.
\end{theorem}

While these results certify density, they do not quantify rates or provide constructive recipes for choosing the parameters in \eqref{eq:finite-sum}. Barron addressed rates by identifying a Fourier-based function class (now called the Barron class) for which there exist networks of the form \eqref{eq:finite-sum} achieving $L^2$ error of order $O(N^{-1/2})$, with constants independent of input dimension \citep{Barron1993}; see also \citet{Pinkus1999} for a survey. Beyond function values, Hornik, Stinchcombe, and White showed that networks can approximate a target mapping together with its partial derivatives to arbitrary accuracy under mild smoothness assumptions \citep{HornikStinchcombeWhite1990}.

These developments establish the expressive power of the finite-sum architecture \eqref{eq:finite-sum} but remain largely nonconstructive regarding \emph{how} to realize a desired decision set. In this paper we remain strictly within the UAT format and propose a geometry-aware procedure that \emph{constructs} the parameters in \eqref{eq:finite-sum} from a prescribed planar region via a ball cover. The resulting sigmoidal networks inherit all classical UAT guarantees by design, while our compilation yields decision boundaries that already align with the target at initialization; training, when used, serves primarily as refinement rather than discovery.
\section{Tropical Geometry}\label{sec:tropical}
Tropical geometry studies algebraic objects after replacing the usual $(+,\times)$ arithmetic by an idempotent, piecewise-linear calculus. Under this correspondence, polynomial equations give rise to polyhedral complexes equipped with rich combinatorics, so that questions from algebraic geometry can be translated into convex- and graph-theoretic statements; see \citet{Morrison2019} for a concise introduction and \citet{MaclaganSturmfels2015} for a comprehensive reference. In this section, we only state som relevant concepts in tropical grometry.

\begin{definition}[Tropical semiring (max--plus)]
\label{def:trop-semiring}
The tropical semiring is $\mathbb{T}=\mathbb{R}\cup\{-\infty\}$ with
\[
a\oplus b := \max\{a,b\}, \qquad
a\odot b := a+b,
\]
additive identity $-\infty$ and multiplicative identity $0$.
\end{definition}

\begin{proposition}[Basic properties of tropical semiring]
\label{prop:basic}
Over $(\mathbb{T},\oplus,\odot)$:
\begin{enumerate}[label=\textnormal{(\alph*)}]
\item $\oplus$ is commutative, associative, and idempotent ($a\oplus a=a$).
\item $\odot$ is commutative and associative with identity $0$.
\item $\odot$ distributes over $\oplus$: $a\odot(b\oplus c)=(a\odot b)\oplus(a\odot c)$.
\item The natural order $a\le b \iff a\oplus b=b$ coincides with the usual order on $\mathbb{R}$.
\end{enumerate}
Hence $\mathbb{T}$ is a commutative idempotent semiring (not a ring: $\oplus$ has no additive inverses). \citep[cf.][]{Morrison2019,MaclaganSturmfels2015}
\end{proposition}

\begin{remark}[Min--plus convention]
\label{rem:minplus}
Many texts use the min--plus semiring $(\overline{\mathbb{R}},\boxplus,\boxtimes)$ with $a\boxplus b=\min\{a,b\}$ and $a\boxtimes b=a+b$, additive identity $+\infty$. It is anti-isomorphic to max--plus via $x\mapsto -x$ \cite{Morrison2019,MaclaganSturmfels2015}. 
\end{remark}

\subsection*{Tropical monomials, polynomials, and hypersurfaces}
\begin{definition}\label{def:inner}
For $u=(u_1,\dots,u_n)\in\mathbb{Z}_{\ge 0}^n$ and $x=(x_1,\dots,x_n)\in\mathbb{T}^n$ we write
\[
\langle u,x\rangle := \sum_{i=1}^n u_i\,x_i ,
\]
the Euclidean dot product. 
\end{definition}

\begin{definition}[Monomials and polynomials]
\label{def:poly}
For $x=(x_1,\dots,x_n)\in\mathbb{T}^n$ and $u=(u_1,\dots,u_n)\in\mathbb{Z}_{\ge0}^n$, the tropical monomial with coefficient $c\in\mathbb{T}$ is
$c\odot x^u = c+\langle u,x\rangle$.
A \emph{tropical polynomial} is a finite tropical sum
\[
F(x)=\bigoplus_{k=1}^{m}\big(c_k\odot x^{u^{(k)}}\big)
=\max_{1\le k\le m}\big\{\,c_k+\langle u^{(k)},x\rangle\,\big\}.
\]
\end{definition}

\begin{remark}[Tropical powers]\label{rem:trop-powers}
For $a\in\mathbb{T}$ and $m\in\mathbb{Z}_{\ge 0}$, the tropical $m$-th power is the $m$-fold tropical product
\[
a^{\odot m} := \underbrace{a\odot a\odot\cdots\odot a}_{m\ \text{times}} \;=\; m\,a \quad(\text{ordinary sum}).
\]
For $x=(x_1,\dots,x_n)$ and a multi-index $u=(u_1,\dots,u_n)$ we set
\[
x^{\odot u} \;:=\; x_1^{\odot u_1}\odot\cdots\odot x_n^{\odot u_n}
\;=\; u_1 x_1 + \cdots + u_n x_n .
\]
For convenience we drop the symbol $\odot$ in exponents and write
\[
a^{u}\equiv a^{\odot u}, \qquad x^{u}\equiv x^{\odot u}.
\]
With this convention, a tropical monomial with coefficient $c$ can be written as
\[
c\odot x^{u} \;=\; c + \langle u,x\rangle .
\]
\end{remark}

\begin{proposition}[Convex piecewise-linear structure]\label{prop:convexPL}
Every tropical polynomial $F:\mathbb{R}^n\to\mathbb{R}$ can be written as
$F(x)=\max_{1\le k\le m}\{c_k+\langle u^{(k)},x\rangle\}$.
Hence $F$ is convex and piecewise-linear with finitely many linear regions, and its set of non-differentiability points equals the corner locus $\mathcal{T}(F)$. \citep{Morrison2019,MaclaganSturmfels2015}
\end{proposition}

\begin{example}[Tropical line in $\mathbb{R}^2$]\label{ex:tropLine}
Let $F(x,y)=x\oplus y\oplus 0=\max\{x,y,0\}$. Then $\mathcal{T}(F)$ is the union of three rays meeting at the origin, with primitive directions $(-1,0)$, $(0,-1)$, and $(1,1)$ \citep{Morrison2019}. See Fig.~\ref{fig:tropical-line}.
\end{example}

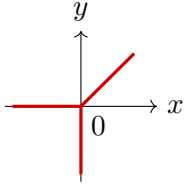
\begin{figure}[H]
  \centering
  \begin{tikzpicture}[scale=0.5]
    \draw[->] (-2,0)--(2,0) node[right] {$x$};
    \draw[->] (0,-2)--(0,2) node[above] {$y$};
    \draw[very thick,draw=red!80!black] (0,0)--(-1.8,0);
    \draw[very thick,draw=red!80!black] (0,0)--(0,-1.8);
    \draw[very thick,draw=red!80!black] (0,0)--(1.4,1.4);
    \node[below right] at (0,0) {$0$};
  \end{tikzpicture}
  \caption{The tropical line $x\oplus y\oplus 0$: three rays at the origin.}
  \label{fig:tropical-line}
\end{figure}

\begin{figure}[H]
\centering
\setlength{\tabcolsep}{6pt}

\newcommand{\TFbox}{\useasboundingbox (-2.25,-2.25) rectangle (2.25,2.25);} 
\newcommand{\DNbox}{\useasboundingbox (-0.85,-0.85) rectangle (3.45,2.85);} 

\tikzset{
  outnorm/.style={
    red,dashed,thick,-{Latex[length=2mm]},
    dash pattern=on 4pt off 3pt,
    line cap=round,line join=round,
    preaction={draw=white,line width=3pt}
  }
}

\begin{minipage}[t]{0.315\textwidth}\centering
\footnotesize \(f_{0}(x,y)=0\oplus x\oplus y\)\\[-0.15em]
\begin{tikzpicture}[scale=0.78]
  \TFbox
  \draw[->] (-1.8,0)--(1.9,0) node[right] {$x$};
  \draw[->] (0,-1.8)--(0,1.9) node[above] {$y$};
  \draw[line width=2pt,red!80!black] (0,0)--(-1.6,0);
  \draw[line width=2pt,red!80!black] (0,0)--(0,-1.6);
  \draw[line width=2pt,red!80!black] (0,0)--(1.15,1.15);
  \fill (0,0) circle (1.9pt);
  \node[below left] at (0,0) {\scriptsize $T(f_{0})$};
\end{tikzpicture}

\vspace{0.15em}
\begin{tikzpicture}[scale=0.88]
  \DNbox
  \draw[->] (-0.65,0)--(2.05,0) node[below] {$u_1$};
  \draw[->] (0,-0.65)--(0,2.05) node[left] {$u_2$};

  \filldraw[fill=blue!60, fill opacity=0.22, draw=blue!60!black, thick]
           (0,0)--(1,0)--(0,1)--cycle;
  \foreach \p/\lab in {(0,0)/{\tiny(0,0)},(1,0)/{\tiny(1,0)},(0,1)/{\tiny(0,1)}}{
    \fill \p circle (1.5pt) node[above right] {\lab};}

  \draw[outnorm] (0.50,0.35) -- (0.50,-0.55); 
  \draw[outnorm] (0.35,0.50) -- (-0.55,0.50); 
  \draw[outnorm] (0.45,0.45) -- (1.25,1.25);  

  \node[below left] at (0,0) {\scriptsize $\delta(f_{0})$};
\end{tikzpicture}
\end{minipage}
\hfill
\begin{minipage}[t]{0.315\textwidth}\centering
\footnotesize \(f_{1}(x,y)=(0\oplus x\oplus y)\odot(0\oplus(x{-}1)\oplus(y{+}1))\)\\[-0.15em]
\begin{tikzpicture}[scale=0.78]
  \TFbox
  \draw[->] (-1.8,0)--(1.9,0) node[right] {$x$};
  \draw[->] (0,-1.8)--(0,1.9) node[above] {$y$};
  \draw[line width=2pt,red!80!black] (0,0)--(-1.6,0);
  \draw[line width=2pt,red!80!black] (0,0)--(0,-1.6);
  \draw[line width=2pt,red!80!black] (0,0)--(1.0,1.0);
  \fill (0,0) circle (1.9pt);
  \draw[line width=2pt,red!80!black] (1,-1)--(-0.6,-1);
  \draw[line width=2pt,red!80!black] (1,-1)--(1,-2.6);
  \draw[line width=2pt,red!80!black] (1,-1)--(1.9,-0.1);
  \fill (1,-1) circle (1.9pt);
  \node at (-1.55,1.55) {\scriptsize $T(f_{1})$};
\end{tikzpicture}

\vspace{0.15em}
\begin{tikzpicture}[scale=0.88]
  \DNbox
  \draw[->] (-0.65,0)--(2.90,0) node[below] {$u_1$};
  \draw[->] (0,-0.65)--(0,2.90) node[left] {$u_2$};

  \filldraw[fill=blue!60, fill opacity=0.22, draw=blue!60!black, thick]
           (0,0)--(2,0)--(0,2)--cycle;
  \draw[blue!60!black] (1,0)--(1,1);
  \draw[blue!60!black] (0,1)--(1,1);
  \foreach \p/\lab in {(0,0)/{\tiny(0,0)},(2,0)/{\tiny(2,0)},(0,2)/{\tiny(0,2)},
                       (1,0)/{\tiny(1,0)},(0,1)/{\tiny(0,1)},(1,1)/{\tiny(1,1)}}{
    \fill \p circle (1.5pt) node[above right] {\lab};}

  \draw[outnorm] (0.50,1.35) -- (0.50,-0.55);
  \draw[outnorm] (1.50,0.45) -- (1.50,-0.55);
  \draw[outnorm] (1.40,0.50) -- (-0.55,0.50);
  \draw[outnorm] (0.45,1.50) -- (-0.55,1.50);
  \draw[outnorm] (0.55,1.45) -- (1.35,2.25);
  \draw[outnorm] (1.45,0.55) -- (2.25,1.35);

  \node[below left] at (0,0) {\scriptsize $\delta(f_{1})$};
\end{tikzpicture}
\end{minipage}
\hfill
\begin{minipage}[t]{0.315\textwidth}\centering
\footnotesize \(f_{2}(x,y)=(0\oplus x\oplus y)\odot(0\oplus(x{-}1)\oplus(y{+}1))\odot(0\oplus(x{+}1)\oplus(y{-}1))\)\\[-0.15em]
\begin{tikzpicture}[scale=0.78]
  \TFbox
  \draw[->] (-1.8,0)--(1.9,0) node[right] {$x$};
  \draw[->] (0,-1.8)--(0,1.9) node[above] {$y$};
  \foreach \ax/\ay in {0/0,1/-1,-1/1}{
    \draw[line width=2pt,red!80!black] (\ax,\ay)--(\ax-1.6,\ay);
    \draw[line width=2pt,red!80!black] (\ax,\ay)--(\ax,\ay-1.6);
    \draw[line width=2pt,red!80!black] (\ax,\ay)--(\ax+1.0,\ay+1.0);
    \fill (\ax,\ay) circle (1.9pt);
  }
  \node at (-1.55,1.55) {\scriptsize $T(f_{2})$};
\end{tikzpicture}

\vspace{0.15em}
\begin{tikzpicture}[scale=0.88]
  \DNbox
  \draw[->] (-0.65,0)--(3.30,0) node[below] {$u_1$};
  \draw[->] (0,-0.65)--(0,3.30) node[left] {$u_2$};

  \filldraw[fill=blue!60, fill opacity=0.22, draw=blue!60!black, thick]
           (0,0)--(3,0)--(0,3)--cycle;
  \draw[blue!60!black] (1,0)--(1,2);
  \draw[blue!60!black] (2,0)--(2,1);
  \draw[blue!60!black] (0,1)--(2,1);
  \draw[blue!60!black] (0,2)--(1,2);
  \foreach \p/\lab in {(0,0)/{\tiny(0,0)},(3,0)/{\tiny(3,0)},(0,3)/{\tiny(0,3)},
                       (1,0)/{\tiny(1,0)},(2,0)/{\tiny(2,0)},(0,1)/{\tiny(0,1)},
                       (0,2)/{\tiny(0,2)},(1,1)/{\tiny(1,1)},(2,1)/{\tiny(2,1)},
                       (1,2)/{\tiny(1,2)}}{
    \fill \p circle (1.4pt) node[above right] {\lab};}

  \draw[outnorm] (0.5,2.45) -- (0.5,-0.60);
  \draw[outnorm] (1.5,1.45) -- (1.5,-0.60);
  \draw[outnorm] (2.5,0.45) -- (2.5,-0.60);

  \draw[outnorm] (2.45,0.50) -- (-0.60,0.50);
  \draw[outnorm] (1.45,1.50) -- (-0.60,1.50);
  \draw[outnorm] (0.45,2.50) -- (-0.60,2.50);

  \draw[outnorm] (0.55,2.45) -- (1.35,3.25);
  \draw[outnorm] (1.50,1.50) -- (2.30,2.30);
  \draw[outnorm] (2.45,0.55) -- (3.25,1.35);

  \node[below left] at (0,0) {\scriptsize $\delta(f_{2})$};
\end{tikzpicture}
\end{minipage}

\caption{\textbf{Tropical hypersurfaces \(\Trop(f)\) and subdivision of its Newton polygons \(\delta(f)\).}
Top: \(\Trop(f)\)—unions of 1, 2, and 3 tropical lines (apices \((0,0)\), \((1,-1)\), \((-1,1)\)).
Bottom: \(\delta(f)\)—Newton triangles of sizes \(1,2,3\). Red dashed lines are
\emph{straight outward primitive normals} (left \((-1,0)\), bottom \((0,-1)\), diagonal \((1,1)\)),
each drawn as a single segment that runs from an interior anchor point, across the edge,
and extends outside the polygon.}
\label{fig:trop-one-row-fixed}
\end{figure}
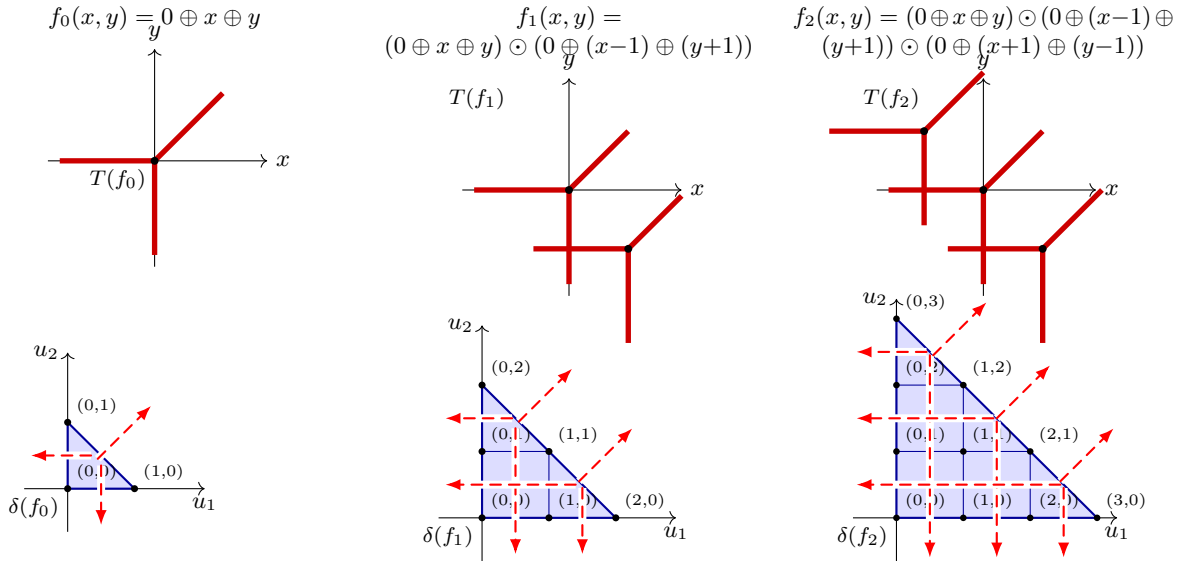


\begin{definition}\label{def:trop-hypersurface}
A tropical polynomial is the tropical sum of finitely many monomials,
\[
F(x)\;=\;\bigoplus_{k=1}^{m}\bigl(c_k\odot x^{\odot u^{(k)}}\bigr)
\;= c_1x^{u^{(1)}} \oplus c_2x^{u^{(2)}} \oplus ... \oplus c_mx^{u^{(m)}}\;=\;\max_{1\le k\le m}\Bigl\{\,c_k+\langle u^{(k)},x\rangle\,\Bigr\},\qquad x\in\mathbb{R}^n.
\]
Its \emph{tropical hypersurface} is
\[
\begin{aligned}
\Trop(F)
&= \Bigl\{\,x\in\mathbb{R}^n:\ \bigl|\Argmax_{k}\{\,c_k \odot x^{u^{(k)}}\,\}\bigr|\ge 2\,\Bigr\} \\
&= \Bigl\{\,x\in\mathbb{R}^n:\ \exists\, i\neq j\ \text{ with }\ 
     c_i \odot x^{\odot u^{(i)}} \;=\; c_j \odot x^{\odot u^{(j)}} \;=\; F(x)\,\Bigr\}.
\end{aligned}
\]
Equivalently, $\Trop(F)$ is the set of non-differentiability points of the convex piecewise-linear map $F$.
For $n=2$ it is a tropical curve.
\citep[cf.][]{Morrison2019,MaclaganSturmfels2015}
\end{definition}

\begin{definition}\label{def:newton}
Let $A=\{u^{(1)},\dots,u^{(m)}\}\subset\mathbb{Z}^n$ be the support of a tropical polynomial $F$, i.e.,
$F=\bigoplus_k(c_k\odot x^{u^{(k)}})$. The \emph{Newton polytope} of $F$ is the convex hull of its support $A$:
\[
\Newt(F)\;:=\;\conv(A)\ = \conv\{u^{(k)}\in \mathbb{R}^n : c_k \neq - \infty, k = 1,2,...,m\} \subset\ \mathbb{R}^n.
\]
When $n=2$, $\Newt(F)$ is a lattice polygon.
\citep{Morrison2019,MaclaganSturmfels2015}
\end{definition}

\begin{definition}[Upper faces, projection, and dual subdivision]\label{def:UF-dual}
Let
\[
F(x)=\bigoplus_{k=1}^{m}\bigl(c_k\odot x^{\odot u^{(k)}}\bigr),\qquad x\in\mathbb{R}^d,
\]
and set $\widehat{A}=\{(u^{(k)},c_k)\}\subset\mathbb{R}^{d+1}$, $P^\wedge=\conv(\widehat{A})$.
Write
\[
\mathrm{UF}(P^\wedge)\ :=\ \bigl\{\,p\subset P^\wedge:\ p\ \text{is a face with an outer normal }\nu\text{ s.t. } \nu_{d+1}>0\,\bigr\}
\]
for the collection of upper faces, and let
\[
\pi:\mathbb{R}^{d}\times\mathbb{R}\longrightarrow\mathbb{R}^{d},\qquad \pi(v,h)=v
\]
be the projection that drops the last coordinate. The \emph{dual subdivision} determined by $F$ is
\[
\delta(F)\ :=\ \{\,\pi(p)\subset\mathbb{R}^{d}\ :\ p\in \mathrm{UF}(P^\wedge)\,\}.
\]
Then $\delta(F)$ is a polyhedral complex with support $\Delta(F):=\Newt(F)=\conv\{u^{(k)}\}$.
By \citet[Prop.~3.1.6]{MaclaganSturmfels2015}, the tropical hypersurface $\Trop(F)$ is the
$(d-1)$\nobreakdash-skeleton of the polyhedral complex dual to $\delta(F)$.
In particular, each vertex of $\delta(F)$ corresponds to one linearity cell of the convex PL map $F$,
so the number of vertices of $\delta(F)$ upper-bounds the number of linear regions of $F$.
\end{definition}

\begin{theorem}[Duality Theorem; {\citep[Prop.~3.1.6]{MaclaganSturmfels2015}}]\label{thm:planar-duality-concise}
Let \(F\) be a tropical polynomial on \(\mathbb{R}^2\) and set \(P=\Newt(F)\).
Let \(\delta(F)=\{\pi(p): p\in \mathrm{UF}(P^\wedge)\}\) be the dual subdivision induced by the coefficients of \(F\)
(as in \Cref{def:UF-dual}). Then the tropical curve \(\Trop(F)\) is dual to \(\delta(F)\) in the following sense:
\begin{itemize}
  \item vertices of \(\Trop(F)\) correspond to polygons in \(\delta(F)\);
  \item edges of \(\Trop(F)\) correspond to interior edges in \(\delta(F)\);
  \item rays of \(\Trop(F)\) correspond to boundary edges of \(P\) used by \(\delta(F)\);
  \item regions of \(\mathbb{R}^2\) separated by \(\Trop(F)\) correspond to lattice points of \(P\) used in \(\delta(F)\).
\end{itemize}
Moreover, dual edges are orthogonal, adjacency is preserved, and dimensions are reversed.
\end{theorem}


\subsection*{Tropical rational functions and neural networks}

\begin{definition}[Tropical rational function]\label{def:trop-rational}
Let $U,V$ be tropical polynomials on $\mathbb{R}^n$. The \emph{tropical rational} function is the classical difference
\[
R(x)\;\equiv\;U(x)\oslash V(x) \equiv U(x) - V(x).
\]
Its zero set $\{x:\,U(x)=V(x)\}$ is a finite union of polyhedral pieces contained in tropical hypersurfaces associated with $U$ and $V$.
\end{definition}

\begin{proposition}\label{prop:zero-equality}
For tropical polynomials $U,V$, the decision set $\{x:\,R(x)\ge 0\}$ coincides with $\{x:\,U(x)\ge V(x)\}$, and the boundary $\partial\{R\ge 0\}$ is contained in the equality locus $\{x:\,U(x)=V(x)\}\subset\mathcal{T}(U)\cup\mathcal{T}(V)$. In particular the boundary is a polyhedral complex of dimension $n-1$.
\end{proposition}

\begin{proposition}[ReLU networks as tropical rational maps]\label{prop:relu-tropical}
Let $f:\mathbb{R}^n\to\mathbb{R}$ be the output of a feedforward network with ReLU activations. Then there exist tropical polynomials $U,V$ such that
\[
f(x)\;=\;U(x) \oslash V(x).
\]
Equivalently, $f$ is a difference of finitely many affine maxima, i.e.,
$f(x)=\big(\max_{i}\{a_i+\langle p_i,x\rangle\}\big)-\big(\max_{j}\{b_j+\langle q_j,x\rangle\}\big)$. Consequently, the linear regions of $f$ form a polyhedral complex, and any binary decision boundary $\{f=0\}$ is a tropical equality set $U=V$. \citep[see][]{ZhangNaitzatLim2018}
\end{proposition}

Proposition~\ref{prop:relu-tropical} provides a polyhedral (tropical) model for ReLU networks. In contrast, the constructions developed later remain strictly within the sigmoidal, finite-sum UAT format, yet borrow the “boundary programming’’ intuition from the tropical viewpoint.

In later sections we compile a ball cover of a target planar region into the weights of a sigmoidal MLP. Although our use is constructive and analytic, unions of balls relate to the \emph{nerve} of a cover; under mild hypotheses, a space is homotopy-equivalent to the nerve of a good cover, providing a combinatorial summary of geometry \citep{Mohnhaupt2023Nerve}.

\section{A one–dimensional warm-up: least–squares construction}
\label{sec:oned}

In one dimension, we adopt the simplest possible realization that still adheres to the finite–sum UAT template: fix a set of shifts (centers) \(p_1,\dots,p_m\in\R\), choose a sigmoidal activation \(\sigma\), and solve a linear least–squares problem for the output weights. Unlike our 2D construction (Sections~\ref{sec:planar-convex}–\ref{sec:planar-general}), no tropical reasoning is needed here—everything reduces to a closed-form linear solve.

\paragraph{Model and design matrix.}
Given data \(\{(x_i,y_i)\}_{i=1}^N\) with \(x_i\in[a,b]\subset\R\) and \(y_i\in\R\), define
\[
\Phi(x)\;=\;\big[\ \sigma(x-p_1),\ \dots,\ \sigma(x-p_m),\ 1\ \big]^\top\in\R^{m+1},
\qquad
\Phi\;=\;\begin{bmatrix}\Phi(x_1)^\top\\ \vdots \\ \Phi(x_N)^\top\end{bmatrix}\in\R^{N\times(m+1)}.
\]
A single-hidden-layer predictor in the UAT finite-sum form is then
\(
\hat y(x)=\Phi(x)^\top \alpha
=\sum_{j=1}^m \alpha_j\,\sigma(x-p_j)+\alpha_{m+1}.
\)

\paragraph{Closed-form weights.}
With squared loss, the optimal output weights solve
\[
\min_{\alpha\in\R^{m+1}} \|\Phi\alpha-y\|_2^2
\quad\Longrightarrow\quad
\alpha^\star=(\Phi^\top\Phi)^{-1}\Phi^\top y,
\]
assuming \(\Phi\) has full column rank. This yields a network that fits the data without any gradient training and stays within the classical UAT format (finite linear combination of shifted/scaled sigmoids of affine inputs).

\paragraph{Binary classification outputs.}
For \(\R\!\to\!\{0,1\}\), one may (i) treat \(\hat y(x)\) as a probability and threshold at \(0.5\), or (ii) add a final logistic squashing \(p(x)=\sigma(\beta_1 \hat y(x)+\beta_0)\) and learn \((\beta_1,\beta_0)\) by 1D logistic regression if desired. In our demos we simply threshold \(\hat y\).

\begin{algorithm}[H]
\caption{Least–squares 1D initializer (training-free)}
\label{alg:1d-ls}
\begin{algorithmic}[1]
\Require Data \(\{(x_i,y_i)\}\); centers \(\{p_j\}_{j=1}^m\); activation \(\sigma\)
\State Build \(\Phi\in\R^{N\times(m+1)}\) with \(\Phi_{ij}=\sigma(x_i-p_j)\) and a final bias column of ones
\State Solve \(\alpha^\star=(\Phi^\top\Phi)^{-1}\Phi^\top y\)
\State Define \(\hat y(x)=\sum_{j=1}^m \alpha^\star_j\,\sigma(x-p_j)+\alpha^\star_{m+1}\)
\State \textbf{return} \(\alpha^\star\), \(\hat y(\cdot)\) \hfill (optional: threshold at \(0.5\))
\end{algorithmic}
\end{algorithm}

\paragraph{Choice of centers and sharpness.}
Placing centers \(\{p_j\}\) near anticipated breakpoints (e.g., interval endpoints) improves accuracy; more centers increase capacity. A larger sigmoid sharpness \(k\) yields sharper transitions at the cost of conditioning—see Fig.~\ref{fig:1d-grid}.

\begin{figure}[H]
\centering
\begin{subfigure}[t]{0.29\textwidth}
    \centering
    \includegraphics[width=\linewidth]{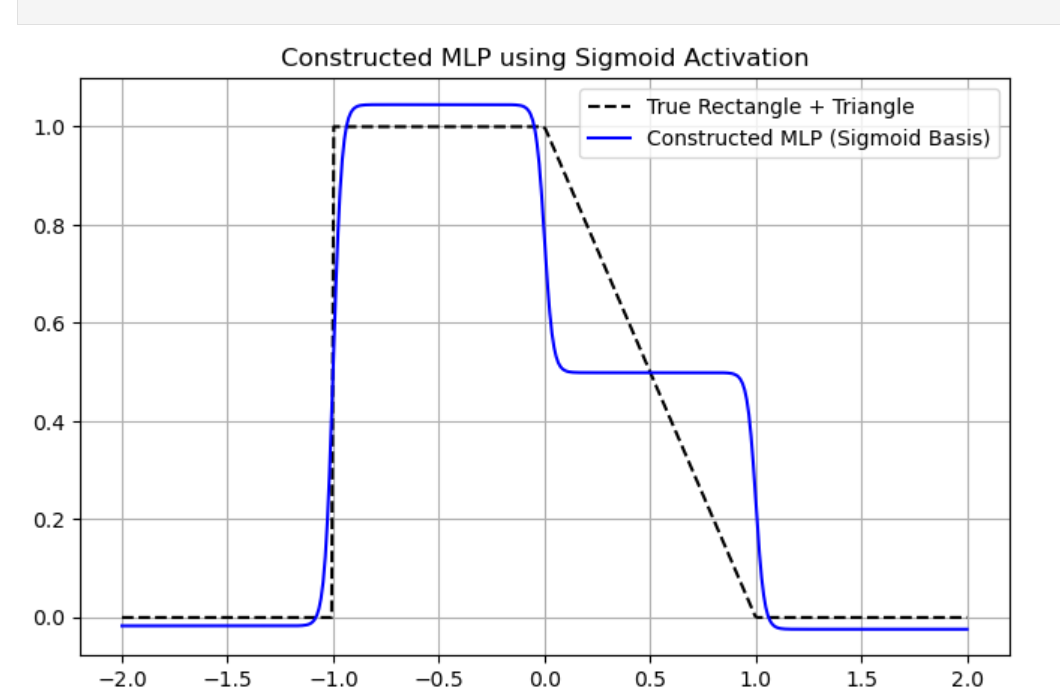}
    \caption*{$k=50$}
\end{subfigure}
\hfill
\begin{subfigure}[t]{0.29\textwidth}
    \centering
    \includegraphics[width=\linewidth]{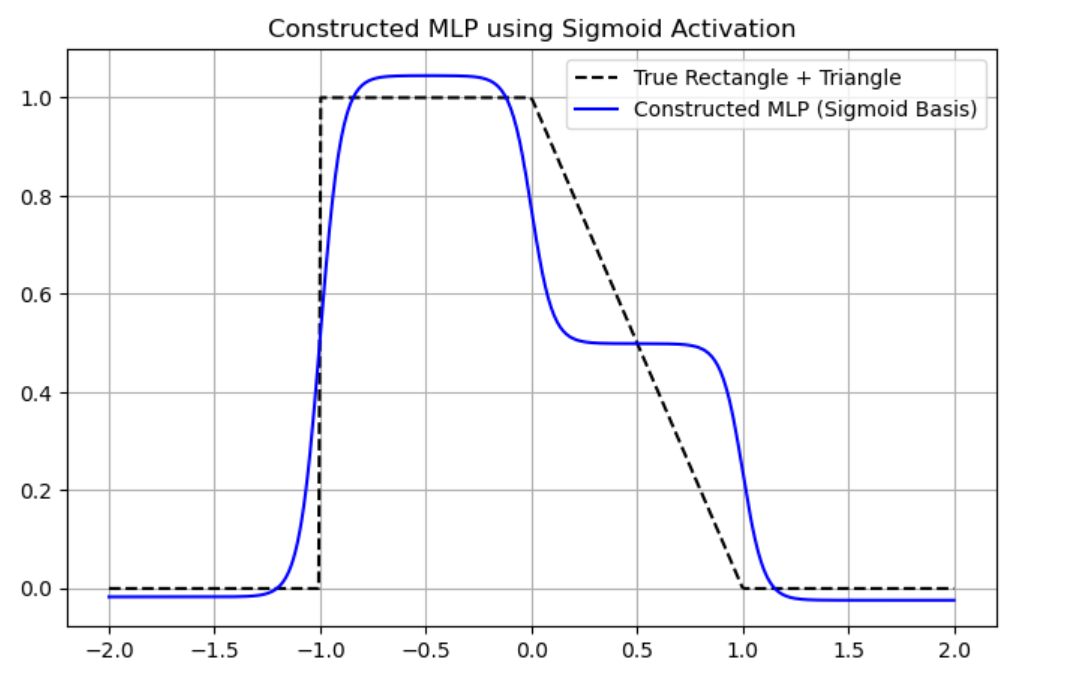}
    \caption*{$k=20$}
\end{subfigure}
\hfill
\begin{subfigure}[t]{0.29\textwidth}
    \centering
    \includegraphics[width=\linewidth]{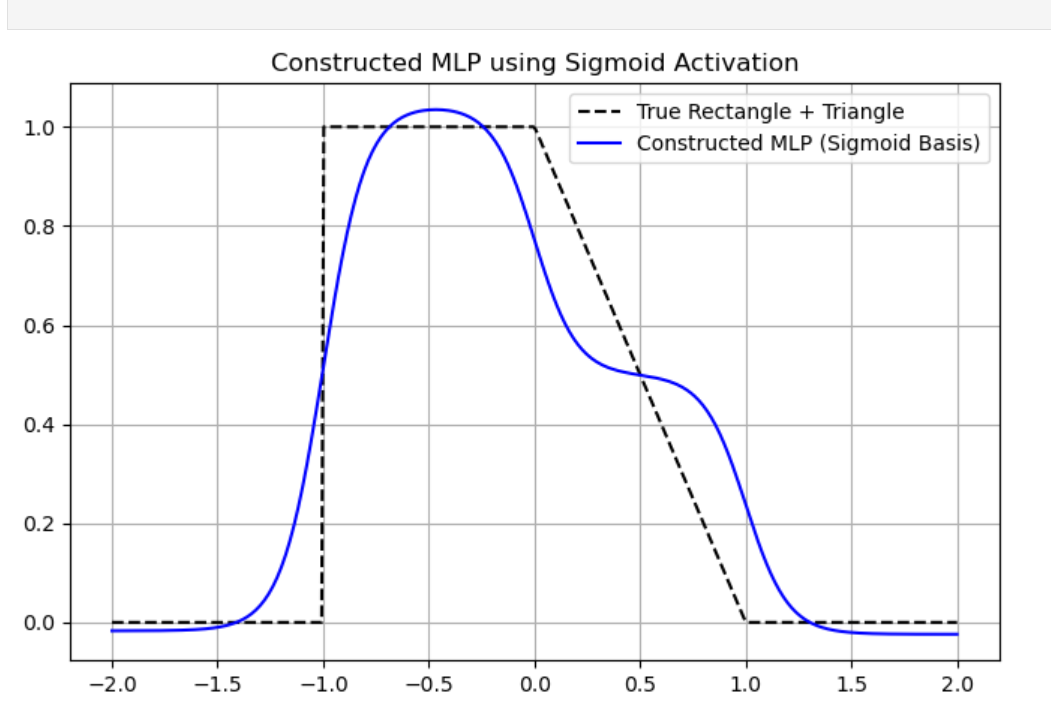}
    \caption*{$k=10$}
\end{subfigure}

\vspace{0.4em}

\begin{subfigure}[t]{0.29\textwidth}
    \centering
    \includegraphics[width=\linewidth]{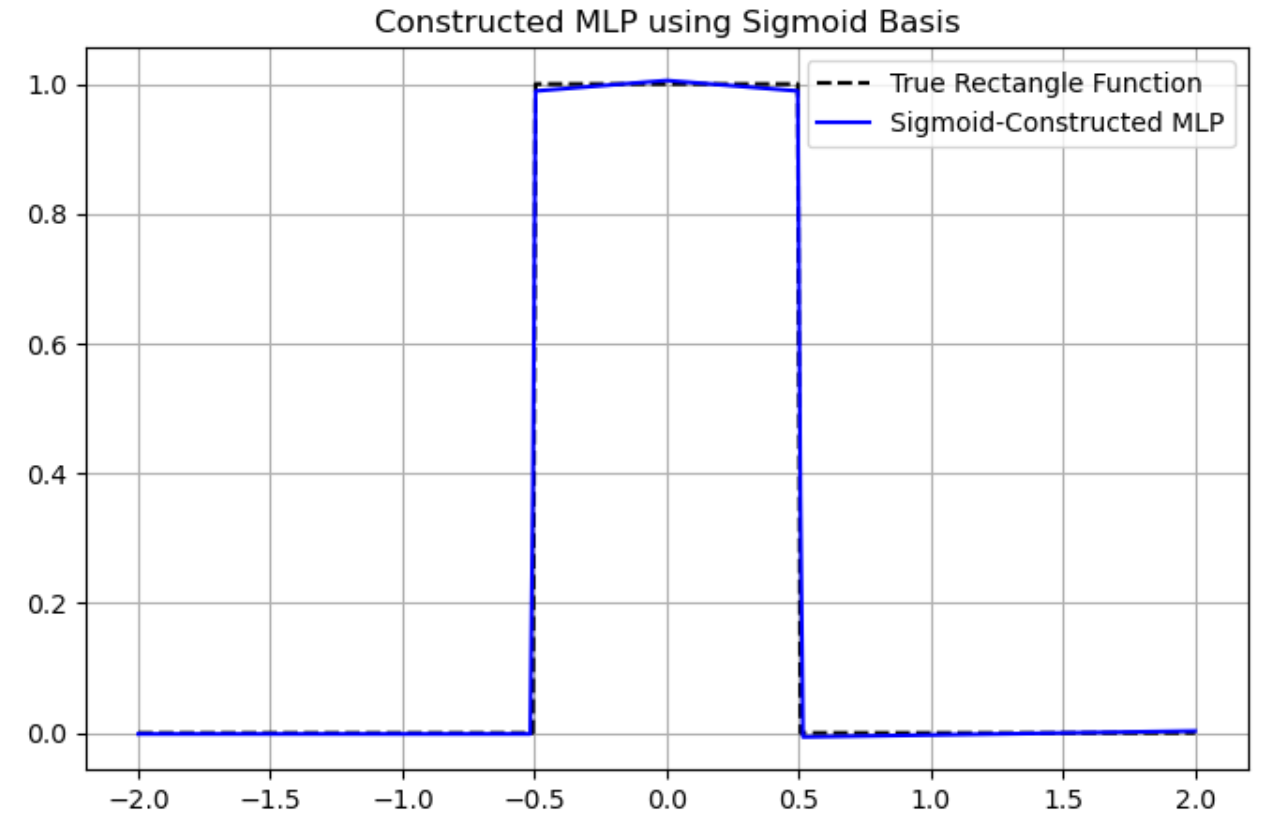}
    \caption*{p=\{-0.51,-0.5,0,0.5,0.51\}}
\end{subfigure}
\hfill
\begin{subfigure}[t]{0.29\textwidth}
    \centering
    \includegraphics[width=\linewidth]{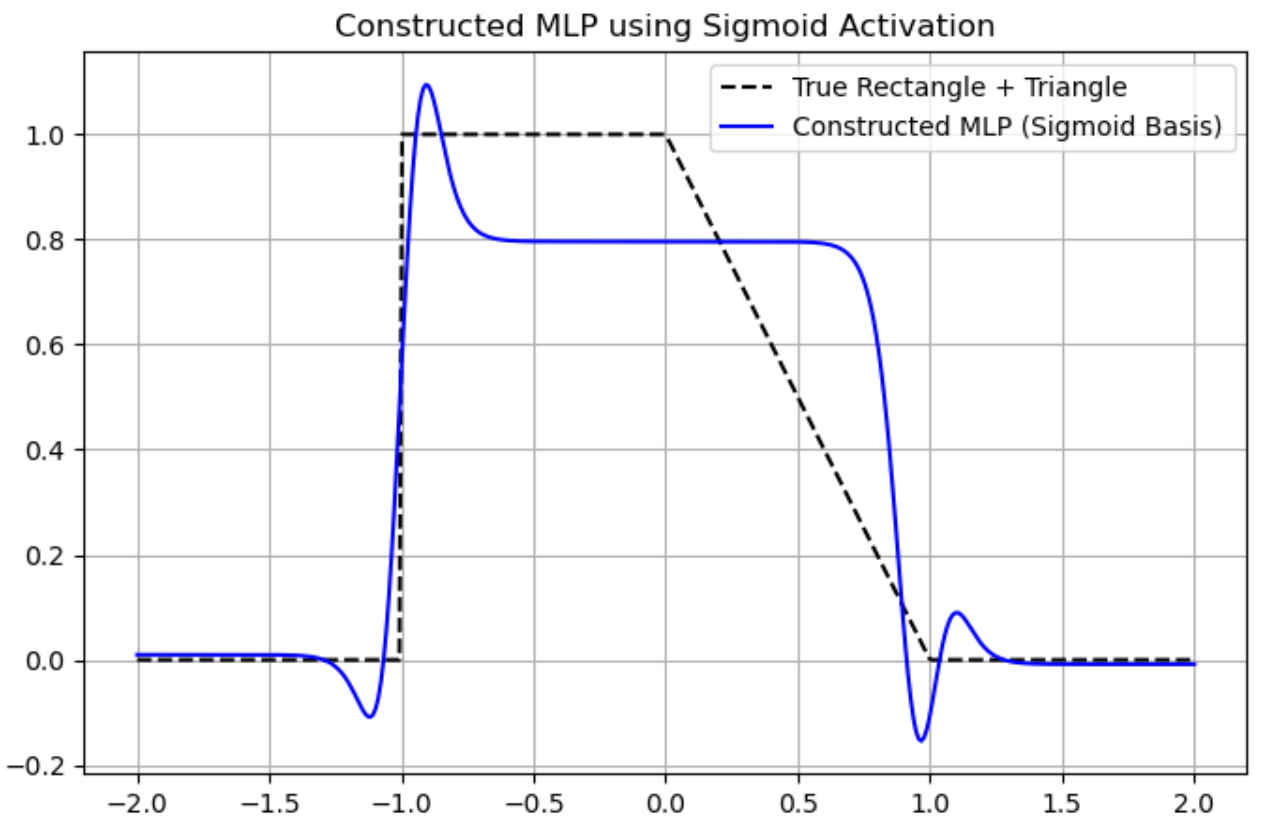}
    \caption*{p=\{-1.1,-1.0,-0.9,0.9,1.0,1.1\}}
\end{subfigure}
\hfill
\begin{subfigure}[t]{0.29\textwidth}
    \centering
    \includegraphics[width=\linewidth]{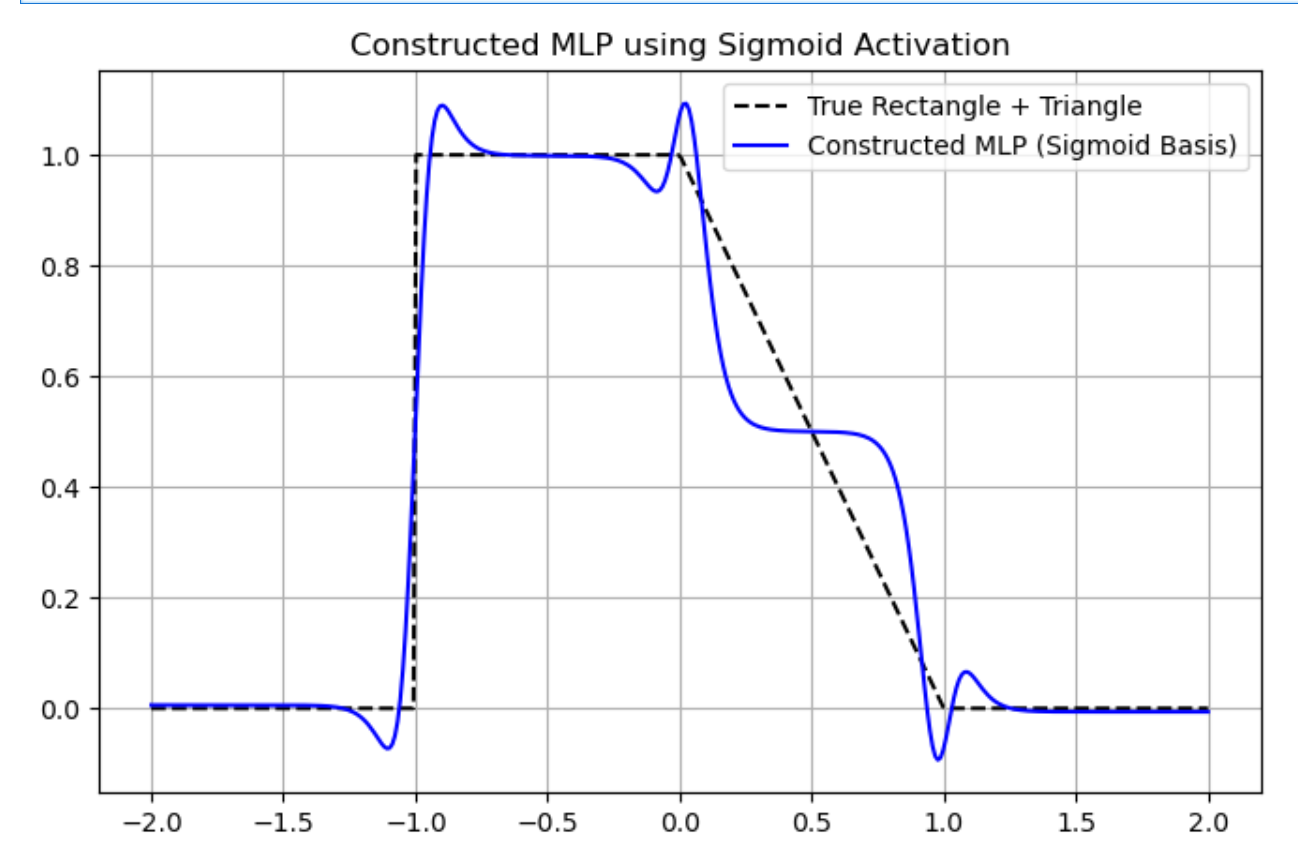}
    \caption*{p=\{-1.01,-1.0,-0.99,-0.01,0\\,0.01,0.99,1.0,1.01\}}
\end{subfigure}

\vspace{0.4em}

\begin{subfigure}[t]{0.29\textwidth}
  \centering\includegraphics[width=\linewidth]{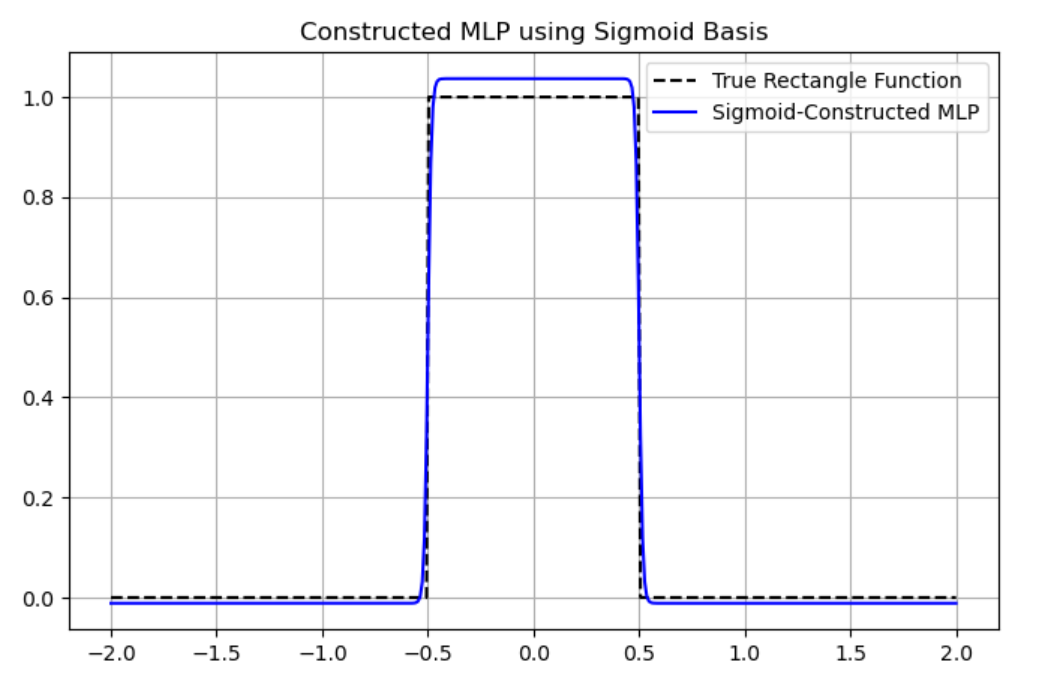}
  \caption*{Rectangle function with Sigmoid(k=120), the given set of points is $\{-0.5, 0.5\}$}
\end{subfigure}
\hfill
\begin{subfigure}[t]{0.29\textwidth}
  \centering\includegraphics[width=\linewidth]{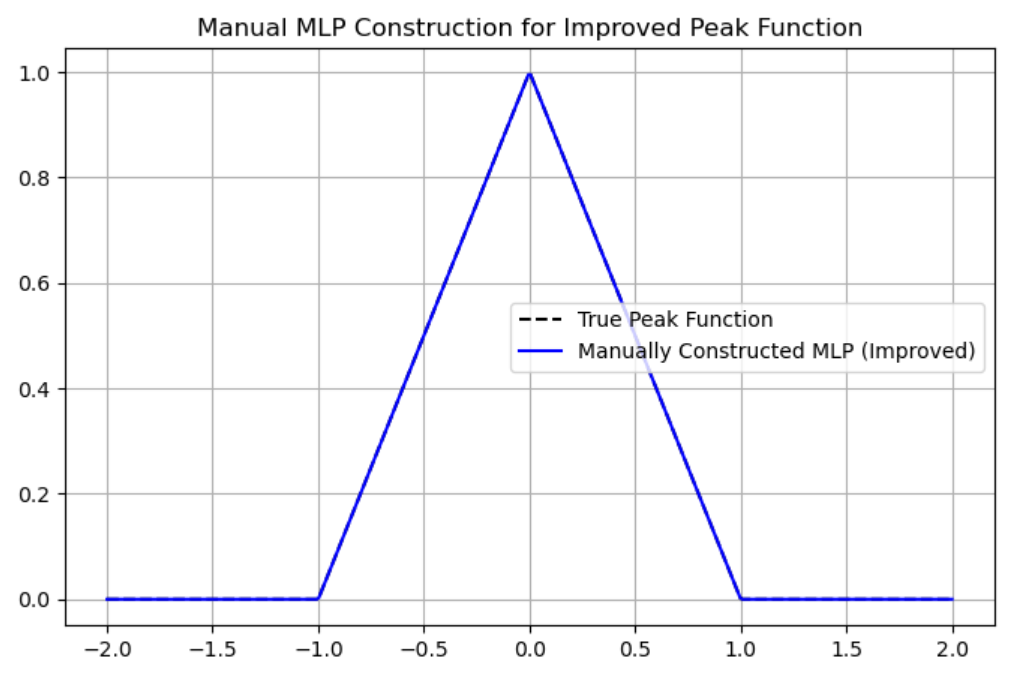}
  \caption*{Triangle function with ReLU, the given set of points is $\{-1, 0.5, 1\}$}
\end{subfigure}
\hfill
\begin{subfigure}[t]{0.29\textwidth}
  \centering\includegraphics[width=\linewidth]{./imgs/ex_trapezoid.png}
  \caption*{Trapezoid function with Sigmoid(k=20), the given set of points is $\{-1, 0.5, 1\}$}
\end{subfigure}

\caption{One–dimensional function fitting with the least–squares initializer.
Top row: varying sigmoid sharpness \(k\).
Middle row: varying center placements \(p_j\) with sigmoid(k=20).
Bottom row: different shapes reconstructed using algorithm \ref{alg:1d-ls} without training.}
\label{fig:1d-grid}
\end{figure}

\paragraph{Summary.}
The 1D case provides a minimal, training-free baseline: once centers are fixed, the network is obtained by a single linear solve. In Section~\ref{sec:planar-convex} we will abandon least-squares in favor of a geometry-aware compilation on the plane (ball covers and differentiable set operations), which is closer in spirit to tropical analyses of decision boundaries while remaining fully sigmoidal.

\section{Planar Convex Regions and Unions}\label{sec:planar-convex}

We now give a constructive, geometry–aware initializer for planar binary
classification \( \R^2\!\to\!\{0,1\} \) that stays within the classical
finite–sum sigmoidal format \eqref{eq:finite-sum}. Geometrically, we
approximate a convex target by supporting half–spaces and place steep sigmoid
“gates’’ on those supports; analytically, we \emph{count} the satisfied
constraints and threshold; tropically, we view the support stencil as a
tropical polynomial whose corner locus prescribes the desired boundary.

\paragraph{Tropical preface.}
Given outward unit normals \(u_\ell\in\mathbb{S}^1\) and supports
\(h_\ell\in\R\) (one per active side), define the tropical polynomial
\[
F(x)\;=\;\bigoplus_{\ell=1}^{m}\bigl(h_\ell \odot \langle -u_\ell,x\rangle\bigr)
\;=\;\max_{1\le \ell\le m}\bigl\{\,h_\ell-\langle u_\ell,x\rangle\,\bigr\}.
\]
Its tropical hypersurface \(\Trop(F)\) is orthogonal to the boundary edges of the
Newton polygon \(\delta(F)\) used by the coefficients, and the set
\(\{F(x)\ge 0\}\) equals the polytope cut out by the corresponding supports.
Our sigmoidal construction below is a smooth relaxation of this hard tropical inequality.


We first recall some basic concepts that are widely used in geometric analysis.

\begin{definition}[Unit circle \(\mathbb{S}^1\)]
\(\mathbb{S}^1:=\{\,u\in\mathbb{R}^2:\ \|u\|_2=1\,\}=\{(\cos\theta,\sin\theta):\theta\in[0,2\pi)\}\).
We write \(u_\ell\in\mathbb{S}^1\) for outward unit normals to supporting lines of a convex set.
\end{definition}

\begin{definition}[Support function and support values]\label{def:support-fn}
For a nonempty compact convex set \(C\subset\mathbb{R}^2\), its support function is
\[
h_C(u)\;:=\;\sup_{x\in C}\ \langle u,x\rangle,\qquad u\in\mathbb{R}^2.
\]
Given a unit outward normal \(u_\ell\in\mathbb{S}^1\), the \emph{support value} is
\(h_\ell:=h_C(u_\ell)=\sup_{x\in C}\langle u_\ell,x\rangle\).
The supporting line with outward normal \(u_\ell\) is \(\{x:\langle u_\ell,x\rangle=h_\ell\}\), and
the containing half-space is \(\{x:\langle u_\ell,x\rangle\le h_\ell\}\).
\end{definition}

\begin{definition}[Hausdorff distance]\label{def:hausdorff}
For nonempty compact \(A,B\subset\mathbb{R}^2\),
\[
d_{\mathrm{H}}(A,B)
=\max\!\Big\{\sup_{a\in A}\inf_{b\in B}\|a-b\|_2,\ \sup_{b\in B}\inf_{a\in A}\|a-b\|_2\Big\}.
\]
Equivalently, \(d_{\mathrm{H}}(A,B)=\inf\{\varepsilon\ge0:\ A\subseteq \mathcal{N}_\varepsilon(B),\ B\subseteq \mathcal{N}_\varepsilon(A)\}\),
where \(\mathcal{N}_\varepsilon(S)=\{x:\operatorname{dist}(x,S)\le\varepsilon\}\).
\end{definition}
Geometrically, the Hausdorff distance measures how close two sets are. In this paper, we will analyze our construction error with respect to this metric.
\paragraph{Tropical interpretation.}
Given outward normals \(\{u_\ell\}_{\ell=1}^m\) and supports \(\{h_\ell\}\),
the tropical polynomial
\[
F(x)\;=\;\bigoplus_{\ell=1}^{m}\bigl(h_\ell \odot  x^{-u_\ell}\bigr)
\;=\;\max_{1\le \ell\le m}\{\,h_\ell-\langle u_\ell,x\rangle\,\}
\]
has tropical hypersurface \(\Trop(F)\) dual to the subdivision of the Newton polygon \(\delta(F)\).
The hard set \(\{F(x)\ge 0\}\) equals the polytope cut out by the supports.

\paragraph{Polyhedral approximation and Gates.}
Let \(C\subset\mathbb{R}^2\) be compact and convex. Choose outward unit normals \(u_\ell\in\mathbb{S}^1\) and
support values \(h_\ell=h_C(u_\ell)\) so that the polytope
\begin{equation}\label{eq:Cpoly-planar}
  C_{\mathrm{poly}}
  \;:=\;
  \bigcap_{\ell=1}^{m}\, \bigl\{\,x\in\mathbb{R}^2:\ \langle u_\ell,x\rangle \le h_\ell\,\bigr\}
\end{equation}
approximates \(C\) to the desired accuracy in Hausdorff distance.
For a sharpness parameter \(\kappa>0\), define the \emph{half–space gates}
\begin{equation}\label{eq:gate-planar}
  s_\ell(x) \;=\; \sigma\!\big(\kappa\, (h_\ell - \langle u_\ell,x\rangle)\big),
  \qquad \ell=1,\dots,m ,
\end{equation}
so that \(s_\ell(x)\approx 1\) on \(\{\langle u_\ell,x\rangle\le h_\ell\}\) and \(s_\ell(x)\approx 0\) otherwise.

\begin{lemma}\label{lem:strip-planar}
Fix \(0<\eta<\tfrac12\) and let \(\sigma(t)=\frac{1}{1+e^{-t}}\) be the logistic sigmoid.
For a gate \(s_\ell(x)=\sigma\!\big(\kappa(h_\ell-\langle u_\ell,x\rangle)\big)\) with \(\kappa>0\),
define the signed distance to the supporting line by
\[
d_\ell(x)\;:=\;h_\ell-\langle u_\ell,x\rangle .
\]
Then the following equivalences hold:
\begin{align}
s_\ell(x)\ \ge\ 1-\eta
\quad &\Longleftrightarrow\quad
d_\ell(x)\ \ge\ w_\eta,
\label{eq:gate-high}\\[0.25em]
s_\ell(x)\ \le\ \eta
\quad &\Longleftrightarrow\quad
d_\ell(x)\ \le\ -\,w_\eta,
\label{eq:gate-low}\\[0.25em]
\eta \ \le\ s_\ell(x)\ \le\ 1-\eta
\quad &\Longleftrightarrow\quad
|d_\ell(x)|\ \le\ w_\eta,
\label{eq:gate-band}
\end{align}
where the half–width
\[
w_\eta\;:=\;\frac{1}{\kappa}\,\log\!\Big(\frac{1-\eta}{\eta}\Big)\,.
\]
Consequently, the \emph{uncertain band} (where the gate is not yet close to \(0\) or \(1\))
around the supporting line \(\langle u_\ell,x\rangle=h_\ell\) is the slab
\(\{x:\ |d_\ell(x)|\le w_\eta\}\) of geometric thickness
\[
2\,w_\eta \;=\; \frac{2}{\kappa}\,\log\!\Big(\frac{1-\eta}{\eta}\Big)\;=\;\Theta(\kappa^{-1}),
\]
with constants depending only on \(\eta\) (not on \(\kappa\)).
\end{lemma}

\begin{proof}
Set \(z:=\kappa\,d_\ell(x)\). Since \(\sigma\) is strictly increasing,
inequalities on \(s_\ell(x)=\sigma(z)\) are equivalent to inequalities on \(z\).
We use the explicit inverse \(\sigma^{-1}(y)=\log\!\big(\frac{y}{1-y}\big)\).

\smallskip\noindent
\emph{(i) High-confidence interior (\(\ge 1-\eta\)).}
\[
s_\ell(x)\ge 1-\eta
\ \Longleftrightarrow\
z\ge \sigma^{-1}(1-\eta)
=\log\!\Big(\frac{1-\eta}{\eta}\Big)
\ \Longleftrightarrow\
\kappa\,d_\ell(x)\ \ge\ \log\!\Big(\frac{1-\eta}{\eta}\Big),
\]
which gives \eqref{eq:gate-high} after dividing by \(\kappa>0\).

\smallskip\noindent
\emph{(ii) High-confidence exterior (\(\le \eta\)).}
\[
s_\ell(x)\le \eta
\ \Longleftrightarrow\
z\le \sigma^{-1}(\eta)
=\log\!\Big(\frac{\eta}{1-\eta}\Big)
=-\,\log\!\Big(\frac{1-\eta}{\eta}\Big)
\ \Longleftrightarrow\
\kappa\,d_\ell(x)\ \le\ -\,\log\!\Big(\frac{1-\eta}{\eta}\Big),
\]
which gives \eqref{eq:gate-low}.

\smallskip\noindent
\emph{(iii) Transition (uncertain) band.}
Combining (i)–(ii),
\[
\eta \le s_\ell(x)\le 1-\eta
\ \Longleftrightarrow\
-\,\log\!\Big(\frac{1-\eta}{\eta}\Big)\ \le\ \kappa\,d_\ell(x)\ \le\
\log\!\Big(\frac{1-\eta}{\eta}\Big),
\]
equivalently \(|d_\ell(x)|\le w_\eta\), which is \eqref{eq:gate-band}.

\smallskip
The set \(\{x:\ |d_\ell(x)|\le w_\eta\}\) is a slab of half–width \(w_\eta\) around the line
\(\langle u_\ell,x\rangle=h_\ell\). Its geometric thickness is therefore \(2w_\eta\).
Since \(\log\!\big(\frac{1-\eta}{\eta}\big)\) is a positive constant for fixed \(\eta\in(0,\tfrac12)\),
we have \(2w_\eta=\Theta(\kappa^{-1})\) as \(\kappa\to\infty\).
\end{proof}

\begin{remark}[Other sigmoidal activations]
If \(\sigma\) is any strictly increasing sigmoidal activation with continuous inverse on \((0,1)\),
the same proof yields
\[
\eta \le s_\ell(x)\le 1-\eta
\ \Longleftrightarrow\
|d_\ell(x)|\ \le\ \frac{1}{\kappa}\,
\max\!\Big\{\sigma^{-1}(1-\eta),\ -\sigma^{-1}(\eta)\Big\},
\]
so the band thickness remains \(\Theta(\kappa^{-1})\) with \(\eta\)-dependent constants.
\end{remark}

\begin{remark}[Roles of \(\eta\) and \(\kappa\)]
The parameter \(\eta\in(0,\tfrac12)\) is a \emph{confidence tolerance}: a gate
\(s_\ell(x)=\sigma(\kappa(h_\ell-\langle u_\ell,x\rangle))\) is \emph{confidently inside}
when \(s_\ell(x)\ge 1-\eta\) and \emph{confidently outside} when \(s_\ell(x)\le\eta\).
This choice fixes the constant
\(\log\!\big(\frac{1-\eta}{\eta}\big)\) and thus the half–width
\(w_\eta=\kappa^{-1}\log\!\big(\frac{1-\eta}{\eta}\big)\) of the ambiguous strip around each
supporting line. In contrast, \(\kappa>0\) is a \emph{sharpness} parameter: increasing
\(\kappa\) shrinks the physical band thickness \(2w_\eta=\Theta(\kappa^{-1})\), concentrating
any classification error into a narrower neighborhood of the boundary. Practically, one may fix a
small \(\eta\) (e.g.\ \(10^{-2}\!\sim\!10^{-1}\)) to interpret the numerical limit to 0/1, then tune \(\kappa\) to
achieve the desired geometric band width. \emph{Equivalently, without the definition of \(\eta\) explicitly, one may target a desired saturation \(s_\ell(x)\ge 1-\epsilon\) at distance \(t\) from the boundary and choose }\(\displaystyle \kappa\ge \frac{1}{t}\log\!\frac{1-\epsilon}{\epsilon}\).
\end{remark}

\subsection*{Binary classifier on convex sets}

Define the \emph{count function}
\[
  F_{m,\kappa}(x)\;=\;\sum_{\ell=1}^m s_\ell(x)
  \;=\;\sum_{\ell=1}^m \sigma\!\big(\kappa\, (h_\ell - \langle u_\ell,x\rangle)\big),
\]
and the baseline threshold \(\tau_m := m-\tfrac12\).
More generally, writing \(\tau=m-\delta\) with \(\delta\in(0,1)\), correct separation is
guaranteed whenever
\[
  \boxed{\, m\eta \;<\; \delta \;<\; 1-\eta \,}.
\]
Intuitively, a larger \(\delta\) (stricter threshold, \(\tau\) closer to \(m\)) demands a smaller
\(\eta\) so interior gates are sufficiently close to \(1\); a smaller \(\delta\) (looser threshold)
can admit points with an unsatisfied facet unless \(\eta\) is also reduced (both effects can be achieved by increasing \(\kappa\)). In this paper we fix \(\tau_m=m-\tfrac12\) (i.e., \(\delta=\tfrac12\)),
which is valid for any \(\eta<\tfrac{1}{2m}\); thereafter \(\kappa\) is chosen so that
\(w_\eta=\kappa^{-1}\log\!\big(\tfrac{1-\eta}{\eta}\big)\) meets the desired geometric tolerance.

\begin{theorem}[Single-layer MLP classifier on convex sets]\label{thm:convex-planar}
Let \(C\subset\mathbb{R}^2\) be compact and convex and let \(K\supset C\) be compact.
For every \(\varepsilon>0\) there exist \(m,\kappa\) such that the decision set
\[
  \Omega_{m,\kappa}\ :=\ \{x\in K:\ F_{m,\kappa}(x)\ge \tau_m\}
\]
satisfies \(d_{\mathrm{H}}(\Omega_{m,\kappa},C) \le \varepsilon\).
Moreover, \(F_{m,\kappa}\) is exactly of the finite–sum form \eqref{eq:finite-sum} with
weights \(w_\ell=-\kappa u_\ell\), biases \(\theta_\ell=\kappa h_\ell\), output weights \(\alpha_\ell=1\) (\(\ell=1,\dots,m\)).
\end{theorem}

\begin{proof}
Pick \(C_{\mathrm{poly}}\) as in \eqref{eq:Cpoly-planar} with \(d_{\mathrm{H}}(C_{\mathrm{poly}},C)\le \varepsilon/2\).
Fix \(0<\eta<\tfrac12\). By Lemma~\ref{lem:strip-planar}, choose \(\kappa\) large so that the union of all uncertainty bands
has Hausdorff thickness \(<\varepsilon/2\).
Then: (i) if \(x\) lies in the interior of \(C_{\mathrm{poly}}\), all \(m\) inequalities hold and
\(F_{m,\kappa}(x)\ge m(1-\eta)>\tau_m\);
(ii) if \(x\notin C_{\mathrm{poly}}\), at least one inequality fails and
\(F_{m,\kappa}(x)\le (m-1)+\eta<\tau_m\).
Thus errors can only occur inside the thin bands, giving \(d_{\mathrm{H}}(\Omega_{m,\kappa},C)\le \varepsilon\).
The finite–sum representation follows from \eqref{eq:gate-planar}.
\end{proof}

\begin{remark}
The hard tropical classifier \(\1\{F(x)\ge 0\}\) with the tropical polynomial
\(F=\bigoplus_{\ell}(h_\ell\odot \langle -u_\ell,x\rangle)=\max_\ell\{h_\ell-\langle u_\ell,x\rangle\}\)
cuts out \(C_{\rm poly}\). Replacing the hard max by the soft transition gates \eqref{eq:gate-planar} and counting satisfied
constraints yields a smooth approximation of the same set. The tropical curve \(\Trop(F)\) prescribes the
piecewise-linear boundary and is dual to the subdivision of the Newton polygon \(\delta(F)\).
\end{remark}

\subsection*{Finite unions of convex sets}

Let \(C=\bigcup_{r=1}^{R} C_r\) with each \(C_r\subset\R^2\) compact and convex.
For component \(r\), fix a supporting half–space description as in
\eqref{eq:Cpoly-planar} with outward unit normals \(\{u_{r,\ell}\}_{\ell=1}^{m_r}\)
and support values \(h_{r,\ell}\in\R\).
Define the \emph{half–space gates} (for a sharpness \(\kappa>0\))
\begin{equation}\label{eq:gate-union}
  s_{r,\ell}(x)\;:=\;\sigma\!\big(\kappa\,(h_{r,\ell}-\langle u_{r,\ell},x\rangle)\big),
  \qquad \ell=1,\dots,m_r,\ \ r=1,\dots,R.
\end{equation}
Let \(0<\eta<\tfrac12\) be fixed; by Lemma~\ref{lem:strip-planar},
each gate has an uncertainty band of thickness \(\Theta(\kappa^{-1})\).
For component \(r\), define the \emph{centered component score}
\begin{equation}\label{eq:component-score}
  J_r(x)\;:=\;\sum_{\ell=1}^{m_r} s_{r,\ell}(x)\;-\;\Bigl(m_r-\tfrac12\Bigr).
\end{equation}
Inside \(C_r\), all constraints are satisfied so \(J_r(x)\) is positive; outside
\(C_r\), at least one constraint is violated and \(J_r(x)\) is negative.

\begin{lemma}\label{lem:margin-one}
Fix \(r\) and \(0<\eta<\tfrac12\). There exists \(\kappa_0\) such that for all
\(\kappa\ge \kappa_0\):
\begin{enumerate}[label=\textnormal{(\alph*)}]
\item If \(x\in \mathrm{int}(C_r)\) and \(\mathrm{dist}(x,\partial C_r)
      \ge t_\eta:=\kappa^{-1}\log\!\big(\tfrac{1-\eta}{\eta}\big)\),
      then \(s_{r,\ell}(x)\ge 1-\eta\) for all \(\ell\) and
      \[
         J_r(x)\ \ge\ m_r(1-\eta) - \Bigl(m_r-\tfrac12\Bigr)
         \;=\; \tfrac12 - m_r\eta.
      \]
\item If \(x\notin C_r\) and \(\mathrm{dist}(x,\partial C_r)\ge t_\eta\),
      then at least one gate satisfies \(s_{r,\ell^\star}(x)\le \eta\), while all
      others are \(\le 1\), hence
      \[
         J_r(x)\ \le\ (m_r-1)+\eta - \Bigl(m_r-\tfrac12\Bigr)
         \;=\; -\tfrac12+\eta.
      \]
\end{enumerate}
\end{lemma}

\begin{proof}
By Lemma~\ref{lem:strip-planar}, if the signed distance to the supporting line
\(\langle u_{r,\ell},x\rangle=h_{r,\ell}\) exceeds \(t_\eta\) on the
\emph{interior} side, then \(s_{r,\ell}(x)\ge 1-\eta\); if it exceeds \(t_\eta\)
on the \emph{exterior} side, then \(s_{r,\ell}(x)\le \eta\). Inside \(C_r\),
every inequality is satisfied, yielding (a). Outside \(C_r\), at least one
inequality is violated, yielding (b).
\end{proof}

To construct an OR-like mechanism for the union of compact sets, we add a second layer.
Pick a second sharpness \(\lambda>0\) and define
\begin{equation}\label{eq:soft-or}
  \Phi_{\kappa,\lambda}(x)
  \;:=\;\sum_{r=1}^{R} \sigma\big(\lambda\,J_r(x)\big)
  \;=\;\sum_{r=1}^{R} \sigma\!\Big(\lambda\,\big(\sum_{\ell=1}^{m_r} s_{r,\ell}(x)\;-\;(m_r-\tfrac12)\big)\Big).
\end{equation}
Set the estimated classifier by thresholding this count at \(\tfrac12\):
\begin{equation}\label{eq:union-classifier}
  \widehat{\1}_C(x)\;=\;\1\!\big\{\,\Phi_{\kappa,\lambda}(x)\ \ge\  \tfrac12\,\big\}.
\end{equation}
Note that (i) the \emph{first layer} consists of all gates \(s_{r,\ell}\) in
\eqref{eq:gate-union}; (ii) the \emph{second layer} forms \(J_r\) by the affine
combination \eqref{eq:component-score} and applies a sigmoid to each \(J_r\);
(iii) the output is an affine sum (with weights \(1\)) thresholded at \(\tfrac12\).
Thus the whole construction is a two–layer sigmoidal MLP.

\begin{theorem}\label{thm:union-planar-robust}
Let \(C=\bigcup_{r=1}^{R} C_r\subset\R^2\) be a finite union of nonempty compact convex sets and let
\(K\supset C\) be compact. Fix \(\varepsilon>0\). For each \(r\), choose outward unit normals \(u_{r,\ell}\in \mathbb{S}^1\) and supports \(h_{r,\ell}\in\R\),
\(\ell=1,\dots,m_r\), so that
\[
C_r=\bigcap_{\ell=1}^{m_r}\{x:\ \langle u_{r,\ell},x\rangle\le h_{r,\ell}\},\qquad
M:=\max_{1\le r\le R} m_r.
\]
Let \(\sigma(t)=\frac{1}{1+e^{-t}}\) be the logistic, strictly increasing with inverse
\(\sigma^{-1}(y)=\log\!\big(\frac{y}{1-y}\big)\).
Fix any \(0<\eta<\min\{\tfrac12,\tfrac{1}{4M}\}\) and define the \(\eta\)-strip half-width
\[
t_\eta\;:=\;\frac{1}{\kappa}\log\!\Big(\frac{1-\eta}{\eta}\Big)\qquad(\kappa>0\ \text{to be chosen}).
\]
For each component \(r\) define the inner gates and the centered score
\[
s_{r,\ell}(x):=\sigma\!\big(\kappa(h_{r,\ell}-\langle u_{r,\ell},x\rangle)\big),\quad
J_r(x):=\sum_{\ell=1}^{m_r}s_{r,\ell}(x)-\Bigl(m_r-\tfrac12\Bigr),
\]
and then the outer “OR-like” aggregator
\[
\Phi_{\kappa,\lambda}(x)\;:=\;\sum_{r=1}^R \sigma\!\big(\lambda\,J_r(x)\big)\qquad(\lambda>0).
\]
Define the decision set
\[
\Omega_{\kappa,\lambda}\;:=\;\bigl\{x\in K:\ \Phi_{\kappa,\lambda}(x)\ge \tfrac12\bigr\}.
\]
Then there exist \(\kappa,\lambda>0\) such that \(d_{\mathrm H}(\Omega_{\kappa,\lambda},C)\le \varepsilon\).
Moreover, for any $\delta\in(0,\tfrac12)$ set $a_\delta:=\log\!\bigl(\tfrac{1-\delta}{\delta}\bigr)$ and define
\[
\mathcal{B}_{\mathrm{in}}
:= \bigcup_{r,\ell}\Bigl\{x:\ \bigl|\,h_{r,\ell}-\langle u_{r,\ell},x\rangle\,\bigr|\ \le t_\eta\Bigr\}
\quad\text{and}\quad
\mathcal{B}_{\mathrm{out}}
:= \bigcup_{r=1}^{R}\Bigl\{x:\ \bigl|J_r(x)\bigr|\ \le \tfrac{a_\delta}{\lambda}\Bigr\}.
\]
Outside the union of these two \emph{uncertainty bands},
\[
x\ \notin\ \mathcal{B}_{\mathrm{in}}\ \cup\ \mathcal{B}_{\mathrm{out}}
\quad\Longrightarrow\quad
\mathbf{1}\!\left\{\Phi_{\kappa,\lambda}(x)\ge \tfrac12\right\}\;=\;1_C(x).
\]
In particular, if $\delta\in\bigl(0,\tfrac{1}{4R}\bigr]$, then the agreement holds uniformly on $K\setminus(\mathcal{B}_{\mathrm{in}}\cup \mathcal{B}_{\mathrm{out}})$, and the band thicknesses satisfy
$\mathrm{thick}(\mathcal{B}_{\mathrm{in}})=\Theta(\kappa^{-1})$ and
$\mathrm{thick}(\mathcal{B}_{\mathrm{out}})=\Theta(\lambda^{-1})$.
\end{theorem}

\begin{proof}
For a fixed gate \(s_{r,\ell}(x)=\sigma(\kappa (h_{r,\ell}-\langle u_{r,\ell},x\rangle))\) set the signed
offset \(d_{r,\ell}(x):=h_{r,\ell}-\langle u_{r,\ell},x\rangle\).
By monotonicity of \(\sigma\) and the explicit inverse \(\sigma^{-1}(y)\), for any \(\eta\in(0,\tfrac12)\):
\begin{align}
s_{r,\ell}(x)\ge 1-\eta \ &\Longleftrightarrow\ d_{r,\ell}(x)\ge \tfrac{1}{\kappa}\log\!\Big(\tfrac{1-\eta}{\eta}\Big)=t_\eta, \label{eq:gate-hi}\\
s_{r,\ell}(x)\le \eta \ &\Longleftrightarrow\ d_{r,\ell}(x)\le -t_\eta, \label{eq:gate-lo}\\
\eta\le s_{r,\ell}(x)\le 1-\eta \ &\Longleftrightarrow\ |d_{r,\ell}(x)|\le t_\eta. 
\end{align}
Thus each inner gate transitions across a slab of geometric thickness \(2t_\eta=\Theta(\kappa^{-1})\).

\textbf{Step 1.}
Define the \emph{interior} and \emph{exterior} cores of \(C_r\) by
\[
C_r^{\mathrm{in}}:=\{x\in C_r:\ \mathrm{dist}(x,\partial C_r)\ge t_\eta\},\qquad
C_r^{\mathrm{ext}}:=\{x\in K\setminus C_r:\ \mathrm{dist}(x,\partial C_r)\ge t_\eta\}.
\]
If \(x\in C_r^{\mathrm{in}}\), then all \(m_r\) inequalities are satisfied with margin at least \(t_\eta\),
hence by \eqref{eq:gate-hi} we have \(s_{r,\ell}(x)\ge 1-\eta\) for every \(\ell\) and
\[
J_r(x)
=\sum_{\ell=1}^{m_r}s_{r,\ell}(x)-\Bigl(m_r-\tfrac12\Bigr)
\ \ge\ m_r(1-\eta)-\Bigl(m_r-\tfrac12\Bigr)
=\tfrac12-m_r\eta
\ \ge\ \tfrac12-M\eta.
\]
If \(x\in C_r^{\mathrm{out}}\), then at least one inequality is violated with exterior margin \(\ge t_\eta\),
so by \eqref{eq:gate-lo} one gate is \(\le \eta\) while the rest are \(\le 1\); hence
\[
J_r(x)\ \le\ (m_r-1)+\eta-\Bigl(m_r-\tfrac12\Bigr)
=\ -\tfrac12+\eta.
\]
With the choice \(\eta\le \frac{1}{4M}\), these yield the uniform bounds
\begin{equation}\label{eq:inner-outer-margins}
x\in C_r^{\mathrm{in}}\ \Rightarrow\ J_r(x)\ge\tfrac14,
\qquad
x\in C_r^{\mathrm{ext}}\ \Rightarrow\ J_r(x)\le-\tfrac14.
\end{equation}

\textbf{Step 2.}
Fix any \(\delta\in(0,\tfrac14]\) and set \(a_\delta:=\sigma^{-1}(1-\delta)=\log\!\big(\frac{1-\delta}{\delta}\big)>0\).
Choose \(\lambda\) so that \(\lambda\, \tfrac14\ge a_\delta\), i.e.
\begin{equation}\label{eq:lambda-choice}
\lambda\ \ge\ 4\,a_\delta.
\end{equation}
Then \(\sigma(\lambda t)\ge 1-\delta\) for all \(t\ge \frac14\) and \(\sigma(\lambda t)\le \delta\) for all \(t\le -\frac14\).
Combining with \eqref{eq:inner-outer-margins} gives
\begin{equation}\label{eq:outer-hi-lo}
x\in C_r^{\mathrm{in}}\ \Rightarrow\ \sigma(\lambda J_r(x))\ge 1-\delta,
\qquad
x\in C_r^{\mathrm{out}}\ \Rightarrow\ \sigma(\lambda J_r(x))\le \delta.
\end{equation}

\textbf{Step 3.}
Let \(C^{\mathrm{in}}:=\bigcup_{r=1}^R C_r^{\mathrm{in}}\) and \(C^{\mathrm{ext}}:=\bigcap_{r=1}^R C_r^{\mathrm{ext}}\).
If \(x\in C^{\mathrm{in}}\), pick \(r^\ast\) with \(x\in C_{r^\ast}^{\mathrm{in}}\). Then, by \eqref{eq:outer-hi-lo},
\[
\Phi_{\kappa,\lambda}(x)
=\sum_{r=1}^R \sigma(\lambda J_r(x))
\ \ge\ \sigma(\lambda J_{r^\ast}(x))
\ \ge\ 1-\delta\ >\ \tfrac12,
\]
so \(x\in\Omega_{\kappa,\lambda}\).
Conversely, if \(x\in C^{\mathrm{out}}\), then \eqref{eq:outer-hi-lo} gives
\(\Phi_{\kappa,\lambda}(x)\le R\,\delta\).
Taking \(\delta\le \tfrac{1}{4R}\) ensures \(\Phi_{\kappa,\lambda}(x)<\tfrac12\), hence \(x\notin\Omega_{\kappa,\lambda}\).

\textbf{Step 4.}
Misclassifications can only occur where the implications used above are not guaranteed, namely
within the union of:
\begin{align*}
\text{(i) inner bands } 
& B_\kappa:=\bigcup_{r,\ell}\Bigl\{x:\ \big|\,h_{r,\ell}-\langle u_{r,\ell},x\rangle\,\big|\le t_\eta\Bigr\},
&&\text{thickness } 2t_\eta=\frac{2}{\kappa}\log\!\Big(\frac{1-\eta}{\eta}\Big)=\Theta(\kappa^{-1});\\
\text{(ii) outer bands }
& Z_{\lambda,\delta}:=\bigcup_{r}\Bigl\{x:\ \big|J_r(x)\big|\le \tfrac{a_\delta}{\lambda}\Bigr\},
&&\text{since }\ |J_r|>\tfrac{a_\delta}{\lambda}\Rightarrow \sigma(\lambda J_r)\in\{\le\delta,\ \ge 1-\delta\}.
\end{align*}
The “transition window’’ in the \emph{argument} of the outer sigmoid is exactly
\([-\tfrac{a_\delta}{\lambda},\tfrac{a_\delta}{\lambda}]\), so the effective geometric thickness of
\(Z_{\lambda,\delta}\) scales as \(\Theta(\lambda^{-1})\).

\textbf{Step 5.}
Since \(K\) is compact, we may choose \(\kappa,\lambda\) large enough so that the Minkowski
(thickness) of \(B_\kappa\cup Z_{\lambda,\delta}\) is \(<\varepsilon\).
Then:
\[
C\subseteq \mathcal N_\varepsilon(\Omega_{\kappa,\lambda})
\quad\text{and}\quad
\Omega_{\kappa,\lambda}\subseteq \mathcal N_\varepsilon(C),
\]
which is equivalent to \(d_{\mathrm H}(\Omega_{\kappa,\lambda},C)\le\varepsilon\).

\medskip
For $x\notin B_\kappa$, each inner gate is saturated: $s_{r,\ell}(x)\in[0,\eta]\cup[1-\eta,1]$.
Hence, for any component $r$,
\[
\begin{aligned}
x\in C_r \ &\Rightarrow\  J_r(x)
=\sum_{\ell=1}^{m_r}s_{r,\ell}(x)-\Bigl(m_r-\tfrac12\Bigr)
\ \ge\ m_r(1-\eta)-\Bigl(m_r-\tfrac12\Bigr)
\ =\ \tfrac12-m_r\eta,\\
x\notin C_r \ &\Rightarrow\  J_r(x)
\ \le\ (m_r-1)\cdot 1+\eta-\Bigl(m_r-\tfrac12\Bigr)
\ =\ \eta-\tfrac12.
\end{aligned}
\]
With $\eta<\min\{\tfrac12,\tfrac{1}{4M}\}$ we get the uniform margin
$J_r(x)\ge \tfrac14$ if $x\in C_r$ and $J_r(x)\le -\tfrac14$ if $x\notin C_r$.
Now, for $x\notin Z_{\lambda,\delta}$ we also have outer saturation:
$s_r(x)=\sigma(\lambda J_r(x))\in[0,\delta]\cup[1-\delta,1]$ with the same split as above.
Therefore:
\[
x\notin C\ \Rightarrow\ \Phi(x)=\sum_{r=1}^R s_r(x)\ \le\ R\delta,\qquad
x\in C\ \Rightarrow\ \Phi(x)\ \ge\ 1-\delta.
\]
Choosing $\delta\le \tfrac{1}{4R}$ yields $\Phi(x)\le \tfrac14<\tfrac12$ for $x\notin C$
and $\Phi(x)\ge 1-\delta\ge \tfrac34>\tfrac12$ for $x\in C$.
Hence the final constructed classifier agrees
with $1_C$ pointwise on $K\setminus\bigl(B_\kappa\cup Z_{\lambda,\delta}\bigr)$, i.e., $\mathbf{1}\{\,\Phi(x)\ge \tfrac12\,\}=1_C(x)$ holds pointwise on
$K\setminus\bigl(B_\kappa\cup Z_{\lambda,\delta}\bigr)$.

\end{proof}

\begin{remark}
The construction realizes a two-layer sigmoidal MLP in the classical finite-sum format.
Layer~1 (all inner gates): weights \(w_{r,\ell}=-\kappa u_{r,\ell}\), biases \(\theta_{r,\ell}=\kappa h_{r,\ell}\),
unit outgoing weights. Layer~2 (per-component centering and squash): for component \(r\),
form \(J_r=\sum_{\ell=1}^{m_r}s_{r,\ell}-(m_r-\tfrac12)\) and apply \(\sigma(\cdot)\).
Output layer: unit weights summing the \(R\) outer activations and threshold at \(1/2\).
\end{remark}


\begin{algorithm}[H]
\caption{Geometry-aware MLP initializer for $\mathbb{R}^2\!\to\!\{0,1\}$}
\label{alg:stencil-planar}
\begin{algorithmic}[1]
\Require Target set: either (i) a convex polygon $C$ with CCW vertices $(v_1,\dots,v_{m})$, or (ii) a finite union $\bigcup_{r=1}^R C_r$ of such polygons with CCW vertices; gate sharpness $\kappa>0$; (optional) second sharpness $\lambda>0$
\For{each component $C_r$ with CCW vertices $(v_{r,1},\dots,v_{r,m_r})$}
  \For{$i=1$ to $m_r$}
    \State Let $e_{r,i} \gets v_{r,i+1}-v_{r,i}$ with $v_{r,m_r+1}\equiv v_{r,1}$
    \State  set $u_{r,i}\gets \dfrac{(e_{r,i})^{\perp_R}}{\|e_{r,i}\|}$ where $(a,b)^{\perp_R}:=(\,b,\,-a\,)$ \Comment{CCW vertices $\Rightarrow$ right-normal is outward}
    \State \textbf{(Support value)} $h_{r,i}\gets \langle u_{r,i},\, v_{r,i}\rangle$
    \State \textbf{(Inner gate / facet classifier)} $s_{r,i}(x)\gets \sigma\!\big(\kappa\,[\,h_{r,i}-\langle u_{r,i},x\rangle\,]\big)$
  \EndFor
  \State \textbf{(Per-component score)} $J_r(x)\gets \sum_{i=1}^{m_r} s_{r,i}(x)\;-\;\bigl(m_r-\tfrac12\bigr)$
  \State \textbf{(Per-component decision gate)} $s_r(x)\gets \sigma\!\big(\lambda\, J_r(x)\big)$ 
\EndFor

\If{$R=1$ \textbf{(single component)}}
  \State \textbf{Classifier:} $\widehat{y}(x)\gets \mathbf{1}\!\left\{\,J_1(x)\ge 0\,\right\}$
\Else
  \State \textbf{OR-like aggregator:} $\Phi_{\kappa,\lambda}(x)\gets \displaystyle\sum_{r=1}^R s_r(x)$
  \State \textbf{Classifier:} $\widehat{y}(x)\gets \mathbf{1}\!\left\{\,\Phi_{\kappa,\lambda}(x)\ge \tfrac12\,\right\}$ 
\EndIf
\State For each facet gate $s_{r,i}(x)=\sigma(\langle w_{r,i},x\rangle+\theta_{r,i})$, set
\[
w_{r,i}=-\,\kappa\,u_{r,i},\qquad \theta_{r,i}= \kappa\,h_{r,i},\qquad \alpha_{r,i}\equiv 1.
\]
\State \textbf{Return} $\{(w_{r,i},\theta_{r,i},\alpha_{r,i})\}$ for all $(r,i)$, together with $(\lambda)$ if using the outer gates.
\end{algorithmic}
\end{algorithm}

\paragraph{Complexity and accuracy.}
If \(\partial C\) is \(C^2\) with bounded curvature, circumscribed
\(m\)-gons achieve \(\mathcal{O}(m^{-2})\) Hausdorff error; the sigmoid band adds
\(\mathcal{O}(\kappa^{-1})\) thickness. Taking \(m=\Theta(\varepsilon^{-1/2})\) and
\(\kappa=\Theta(\varepsilon^{-1})\) yields overall error \(\mathcal{O}(\varepsilon)\).
Network size is \(N=\Theta(m)\) for a single component and
\(N=\Theta(\sum_r m_r)\) for unions.

\begin{remark}[Why \(\tau=\tfrac12\min_r m_r\)?]
Each component \(C_r\) has \(m_r\) gates. Deep inside \(C_r\), all its gates contribute \(\approx 1\), so the component’s
sum is \(\approx m_r\). Outside \(C_r\), at least one gate contributes \(\approx 0\),
so the component’s sum is \(\lesssim m_r-1\). We do not know \emph{which}
component contains a given interior point of the union; therefore we pick a
single threshold that any interior point can pass using only the \emph{smallest}
component. Setting
\[
\tau=\tfrac12\,\min_r m_r
\]
guarantees: (i) if \(x\in C_r\) and is away from \(\partial C_r\), then the single
component \(C_r\) already yields \(\sum_{\ell=1}^{m_r}s_{r,\ell}(x)\ge 2\tau\), so \(x\)
is accepted; (ii) if \(x\notin \bigcup_r C_r\), then every component is missing
at least one satisfied facet and the total score stays below \(\tau\) once the
gates are steep (large \(\kappa\)), hence \(x\) is rejected. Thus
\(\tau=\tfrac12\min_r m_r\) is a conservative, component-agnostic threshold that
makes a single “active’’ piece sufficient for acceptance while preventing spurious
acceptance outside the union.
\end{remark}

\begin{remark}[Tropical summary]
The parameters \(\{(u_\ell,h_\ell)\}\) induces a tropical polynomial
\(F(x)=\bigoplus_\ell (h_\ell\odot\langle -u_\ell,x\rangle)\) with curve
\(\Trop(F)\) dual to \(\delta(F)\). Our initializer replaces the hard set
\(\{F\ge 0\}\) by soft gates \eqref{eq:gate-planar} and implements the decision
\(\1\{F_{m,\kappa}\ge \tau_m\}\). Thus the learned boundary starts aligned with
\(\Trop(F)\) (up to a band of width \(\Theta(\kappa^{-1})\)) and can be refined by
standard training without leaving the finite–sum form in \eqref{eq:finite-sum}.
\end{remark}

\section{General Planar Regions}\label{sec:planar-general}

Section~\ref{sec:planar-convex} constructed a two–layer sigmoidal MLP for convex targets and for finite
unions of convex sets. We now extend the construction to \emph{arbitrary nonconvex} planar targets by
reducing them to finite unions of disks. The only assumption is compactness, which gives
both existence of finite subcovers and Hausdorff–metric control.

\paragraph{Finite ball covers from Heine--Borel.}
Let $C\subset\R^2$ be nonempty and compact. By Heine--Borel, every open cover of $C$ admits a finite subcover.
In particular, for any $\varepsilon>0$ there exist centers and radii $\{(c_j,r_j)\}_{j=1}^{R_B}$ such that
\[
  C\ \subseteq\ \bigcup_{j=1}^{R_B} \overline{B}(c_j,r_j)
  \ \subseteq\ \mathcal N_{\varepsilon}(C),
\]
so the finite union of \emph{closed} disks $\varepsilon$–covers $C$ in the Hausdorff sense. This reduces the general (possibly highly nonconvex) $C$ to a finite
union of convex pieces.

For each disk $\overline{B}(c_j,r_j)$ choose a circumscribed regular $m_j$–gon
\[
  P_j\;=\;\bigcap_{\ell=1}^{m_j}\{\,x:\ \langle u_{j,\ell},x\rangle\le h_{j,\ell}\,\},
  \qquad
  u_{j,\ell}=(\cos\tfrac{2\pi\ell}{m_j},\,\sin\tfrac{2\pi\ell}{m_j}),\ \
  h_{j,\ell}=\langle u_{j,\ell},c_j\rangle+r_j,
\]
so that $d_{\mathrm H}(P_j,\overline{B}(c_j,r_j))=\mathcal O(r_j m_j^{-2})$.
Now apply \S\ref{sec:planar-convex} paradigm: attach to each half–space a sharp sigmoid gate
$s_{j,\ell}(x)=\sigma(\kappa[h_{j,\ell}-\langle u_{j,\ell},x\rangle])$, center the per–polygon sum
$J_j(x)=\sum_{\ell=1}^{m_j}s_{j,\ell}(x)-(m_j-\tfrac12)$, squash $J_j$ once more
$s_j^{\rm out}(x)=\sigma(\lambda J_j(x))$, and aggregate with a OR-like mechanism $\Phi_{\kappa,\lambda}(x)=\sum_{j=1}^{R_B} s_j^{\rm out}(x)$, classifying by
$\widehat{\1}_C(x)=\1\{\Phi_{\kappa,\lambda}(x)\ge \tfrac12\}$.
This is exactly the two–layer sigmoidal MLP architecture already established for finite unions of convex sets(see \ref{thm:union-planar-robust}).

\begin{theorem}\label{thm:ball-compact-final}
Let $C\subset\R^2$ be nonempty and compact. For every $\varepsilon>0$ there exist a finite ball cover
$\{\overline{B}(c_j,r_j)\}_{j=1}^{R_B}$, polygon side counts $\{m_j\}$, and slopes $\kappa,\lambda>0$
such that the decision set $\widehat{C}:=\{x:\widehat{\1}_C(x)=1\}$ produced by the above MLP satisfies
\[
  d_{\mathrm H}(\widehat{C},C)\ \le\ \varepsilon.
\]
Moreover, away from two thin “uncertainty bands’’—the inner facet bands of thickness
$\mathcal O(\kappa^{-1})$ and the outer per–ball bands of thickness $\mathcal O(\lambda^{-1})$—
the classifier agrees pointwise with $1_C$ (same margins as discussed in \ref{thm:union-planar-robust}).
\end{theorem}

\begin{proof}
Fix $\varepsilon>0$, 

\emph{(i) Finite cover (Heine--Borel).}
Since $C$ is compact, by Heine--Borel there exists a finite family of closed disks
$\{\overline{B}(c_j,r_j)\}_{j=1}^{R_B}$ with
\(
C\subseteq \bigcup_{j}\overline{B}(c_j,r_j)\subseteq \mathcal N_{\varepsilon_{\rm cov}}(C).
\)

\emph{(ii) Disk $\to$ polygon.}
For each $j$, let $P_j$ be a circumscribed regular $m_j$–gon:
\(
P_j=\bigcap_{\ell=1}^{m_j}\{x:\langle u_{j,\ell},x\rangle\le h_{j,\ell}\},
\)
with $u_{j,\ell}=(\cos\frac{2\pi \ell}{m_j},\sin\frac{2\pi \ell}{m_j})$
and $h_{j,\ell}=\langle u_{j,\ell},c_j\rangle+r_j$.
Choose $m_j$ large enough so that
\(d_{\mathrm H}(P_j,\overline{B}(c_j,r_j))\le \varepsilon_{\rm poly}\) for all $j$
(e.g.\ using the bound $r_j(1-\cos(\pi/m_j))=\mathcal O(r_j m_j^{-2})$).
Then
\(
\bigcup_j P_j \subseteq \bigcup_j \overline{B}(c_j,r_j)\subseteq \mathcal N_{\varepsilon_{\rm cov}}(C)
\)
and
\(
\bigcup_j \overline{B}(c_j,r_j) \subseteq \mathcal N_{\varepsilon_{\rm poly}}(\,\bigcup_j P_j\,).
\)
Thus
\(
d_{\mathrm H}(\,\bigcup_j P_j,\ C\,)\le \varepsilon_{\rm cov}+\varepsilon_{\rm poly}.
\)

\emph{(iii) Compile $\bigcup_j P_j$ by Theorem~\ref{thm:union-planar-robust}.}
Apply the union-of-convexes construction (Section~\ref{sec:planar-convex}) to the family $\{P_j\}$.
By Theorem~\ref{thm:union-planar-robust}, there exist slopes $\kappa,\lambda>0$ such that:
(a) outside the inner facet bands of half-width
\(
t_{\rm in}=\kappa^{-1}\log\!\big(\tfrac{1-\eta}{\eta}\big)
\)
and the outer per–ball bands of half-width
\(
t_{\rm out}=\lambda^{-1}\log\!\big(\tfrac{1-\delta}{\delta}\big),
\)
the classifier agrees pointwise with $1_{\cup_j P_j}$ with fixed positive margins; and
(b) the geometric thickness of these bands scales as
\(
\mathcal O(\kappa^{-1})\ \text{and}\ \mathcal O(\lambda^{-1}),
\)
respectively. Choose $\kappa,\lambda$ so that the total thickness of the inner bands is $\le \varepsilon_{\rm in}$
and that of the outer bands is $\le \varepsilon_{\rm out}$, which is possible by taking $\kappa,\lambda$ sufficiently large.

\emph{(iv) Hausdorff.}
Let $\widehat{C}$ be the decision set. From (iii) we have
\(
d_{\mathrm H}(\widehat{C},\ \bigcup_j P_j)\ \le\ \varepsilon_{\rm in}+\varepsilon_{\rm out}.
\)
Combining with (ii) gives
\[
d_{\mathrm H}(\widehat{C},C)
\ \le\
d_{\mathrm H}(\widehat{C},\textstyle\bigcup_j P_j)\;+\;d_{\mathrm H}(\textstyle\bigcup_j P_j,\ C)
\ \le\ 
(\varepsilon_{\rm in}+\varepsilon_{\rm out})+(\varepsilon_{\rm cov}+\varepsilon_{\rm poly})
\ =:\ \varepsilon.
\]
Finally, the pointwise agreement away from the two bands and the margin values are exactly those
stated in Theorem~\ref{thm:union-planar-robust}, applied componentwise to $\{P_j\}$.
\end{proof}

\paragraph{The Nerve Theorem.}
Although the network construction above is entirely analytic, the same ball cover also yields a concise \emph{topological} summary. Inflate each disk to an open set $U_j:=B(c_j,r_j+\rho)$ for small $\rho>0$. Then $\mathcal U=\{U_j\}$ is a good open cover (all finite intersections are convex, hence contractible), so the Nerve Theorem gives a homotopy equivalence
\[
  \mathrm{Nerve}(\mathcal U)\ \simeq\ \bigcup_{j=1}^{R_B} U_j
  \qquad \text{\citep{Hatcher2002,EdelsbrunnerHarer2010}}.
\]
Under mild geometric conditions (e.g., $C$ has positive reach), sufficiently small offsets $\mathcal N_\rho(C)$ deformation–retract onto $C$, whence for small enough $\rho$ also $\mathrm{Nerve}(\mathcal U)\simeq C$ \citep{NiyogiSmaleWeinberger2008}. This mirrors the logic in manifold learning (e.g., UMAP), where good covers and Čech nerves certify that the simplicial model captures the correct homotopy type at an appropriate scale \citep{UMAP2018}. In short, \emph{the cover that drives our MLP approximation simultaneously provides a simplicial certificate of the target’s topology}—with no changes to the network.

\begin{algorithm}[H]
\caption{Ball cover $\Rightarrow$ two–layer sigmoidal MLP}
\label{alg:ball-mlp}
\begin{algorithmic}[1]
\Require Compact window $K\supset C$; tolerance split $(\varepsilon_{\rm cov},\varepsilon_{\rm poly})$; sharpness $(\kappa,\lambda)$ \emph{or} margin targets $(\eta,\delta)$ with $\eta<\frac{1}{4M}$ ($M:=\max_j m_j$) and $\delta\le \frac{1}{4R_B}$
\Ensure A two–layer sigmoidal MLP with decision set $\widehat{C}$ and $d_{\mathrm H}(\widehat{C},C)\lesssim \varepsilon_{\rm cov}+\varepsilon_{\rm poly}+\mathcal O(\kappa^{-1})+\mathcal O(\lambda^{-1})$
\State \textbf{Finite ball cover:} Choose $\{\overline{B}(c_j,r_j)\}_{j=1}^{R_B}$ such that $C\subseteq \bigcup_j \overline{B}(c_j,r_j)\subseteq \mathcal N_{\varepsilon_{\rm cov}}(C)$.
\State \textbf{Polygonize each ball:} For disk $j$, pick $m_j$ and angles $\theta_{j,\ell}=2\pi\ell/m_j$; set $u_{j,\ell}=(\cos\theta_{j,\ell},\sin\theta_{j,\ell})$ and $h_{j,\ell}=\langle u_{j,\ell},c_j\rangle+r_j$. Then
$P_j=\bigcap_{\ell}\{x:\langle u_{j,\ell},x\rangle\le h_{j,\ell}\}$ with
$d_{\mathrm H}(P_j,\overline{B}(c_j,r_j))=\mathcal O(r_j m_j^{-2})$.
\State \textbf{Inner gates:} $s_{j,\ell}(x)=\sigma\!\big(\kappa\,[h_{j,\ell}-\langle u_{j,\ell},x\rangle]\big)$, \quad $J_j(x)=\sum_{\ell=1}^{m_j}s_{j,\ell}(x)-(m_j-\tfrac12)$.
\State \textbf{Outer gates (OR-like aggregator):} $s_j^{\rm out}(x)=\sigma(\lambda J_j(x))$, \quad $\Phi(x)=\sum_{j=1}^{R_B}s_j^{\rm out}(x)$, \quad $\widehat{\1}_C(x)=\1\{\Phi(x)\ge \tfrac12\}$.
\State \textbf{Weights:} Inner: $w_{j,\ell}=-\kappa u_{j,\ell}$, $\theta_{j,\ell}=\kappa h_{j,\ell}$, output weight $=1$. Outer (per ball): affine map $J_j$ followed by $\sigma(\lambda\cdot)$; final layer sums with unit weights.
\end{algorithmic}
\end{algorithm}

\paragraph{Complexity}
Fix $\varepsilon>0$ and split the budget as
$\varepsilon_{\rm cov}+\varepsilon_{\rm poly}+c_1\kappa^{-1}+c_2\lambda^{-1}\le\varepsilon$,
with $c_1=\log\!\frac{1-\eta}{\eta}$ and $c_2=\log\!\frac{1-\delta}{\delta}$ (inner/outer margins; $\eta<\tfrac{1}{4M}$, $M:=\max_j m_j$, $\delta\le\tfrac{1}{4R_B}$). For disk $j$ (radius $r_j$), choose
$m_j=\Theta\!\big(\sqrt{r_j/\varepsilon_{\rm poly}}\big)$ so $d_{\mathrm H}(P_j,\overline{B})=\mathcal O(r_j m_j^{-2})\le\varepsilon_{\rm poly}$, hence the total inner units are
$N_{\rm inner}=\sum_j m_j=\Theta\!\big(\sum_j \sqrt{r_j/\varepsilon_{\rm poly}}\big)$ and the outer units equal $R_B$; parameters and per–point cost are $\Theta(N_{\rm inner})$ (since $m_j\gg1$). The Hausdorff error obeys
$d_{\mathrm H}(\widehat{C},C)\le \varepsilon_{\rm cov}+\varepsilon_{\rm poly}+c_1\kappa^{-1}+c_2\lambda^{-1}$.

\section{Beyond the Plane: Higher-Dimensional Extensions}\label{sec:higher-d}

This section extends the constructions of \S\ref{sec:planar-convex}–\ref{sec:planar-general} from the plane to $\mathbb{R}^d$ ($d\ge2$) without changing the network architecture in the form \eqref{eq:finite-sum}. The pointwise guarantees, Hausdorff control, and complexity decomposition remain the same; only the construction of covers is replaced by ball$\to$polytope with the standard high–dimensional approximation rate.

\paragraph{Ball $\to$ polytope in $\mathbb{R}^d$.}
Let $\overline B(c,r)\subset\R^d$ be a closed Euclidean ball. An $m$–facet circumscribed polytope $P$ satisfies
\[
d_{\mathrm H}\big(P,\overline B(c,r)\big)\;=\;\mathcal O\!\big(r\,m^{-2/(d-1)}\big).
\]
Hence choosing
\[
m_j\;=\;\Theta\!\Big(\big(\tfrac{r_j}{\varepsilon_{\rm poly}}\big)^{\frac{d-1}{2}}\Big)
\quad\Rightarrow\quad
d_{\mathrm H}\big(P_j,\overline B(c_j,r_j)\big)\le \varepsilon_{\rm poly}.
\]

\begin{theorem}[Union of convex sets in $\R^d$ (high-dimensional version of Thm.\ref{thm:union-planar-robust})]
\label{thm:union-hd}
Let $C=\bigcup_{r=1}^{R} C_r\subset\R^d$ be a finite union of nonempty compact convex sets and let $K\supset C$ be compact.
For component $r$, fix outward unit normals $u_{r,\ell}\in\mathbb S^{d-1}$ and supports $h_{r,\ell}\in\R$ ($\ell=1,\dots,m_r$), and define inner gates and centered score
\[
s_{r,\ell}(x):=\sigma\!\big(\kappa\,[h_{r,\ell}-\langle u_{r,\ell},x\rangle]\big),\qquad
J_r(x):=\sum_{\ell=1}^{m_r}s_{r,\ell}(x)-\Bigl(m_r-\tfrac12\Bigr).
\]
Let the “OR-like” aggregator be
\[
\Phi_{\kappa,\lambda}(x)\;:=\;\sum_{r=1}^{R}\sigma\big(\lambda\,J_r(x)\big),\qquad
\widehat{\1}_C(x):=\1\!\big\{\Phi_{\kappa,\lambda}(x)\ge \tfrac12\big\}.
\]
Fix $0<\eta<\min\{\tfrac12,\tfrac{1}{4M}\}$ with $M:=\max_r m_r$, and $0<\delta\le\tfrac{1}{4R}$, and set
\[
c_1:=\log\!\frac{1-\eta}{\eta},\qquad c_2:=\log\!\frac{1-\delta}{\delta}.
\]
Then there exist $\kappa,\lambda>0$ such that, outside the two uncertainty bands
\[
\mathcal B_{\rm in}:=\bigcup_{r,\ell}\Bigl\{x:\ \bigl|h_{r,\ell}-\langle u_{r,\ell},x\rangle\bigr|\le \tfrac{c_1}{\kappa}\Bigr\},\qquad
\mathcal B_{\rm out}:=\bigcup_{r=1}^{R}\Bigl\{x:\ |J_r(x)|\le \tfrac{c_2}{\lambda}\Bigr\},
\]
we have \(\1\{\Phi_{\kappa,\lambda}(x)\ge \tfrac12\}=1_C(x)\). Moreover,
\[
\mathrm{thick}(\mathcal B_{\rm in})=\Theta(\kappa^{-1}),\qquad
\mathrm{thick}(\mathcal B_{\rm out})=\Theta(\lambda^{-1}).
\]

\end{theorem}
\begin{proof}
    The proof is similar to the planar case: Lemma~\ref{lem:strip-planar} is dimension–free, so the inner/outer margin arguments and the saturation outside bands carry over verbatim.
\end{proof}
\begin{theorem}[Compact sets via ball covers in $\R^d$ (high-dimensional version of Thm.\ref{thm:ball-compact-final})]
\label{thm:ball-hd}
Let $C\subset\R^d$ be nonempty and compact. For every $\varepsilon>0$ there exist a finite ball cover $\{\overline B(c_j,r_j)\}_{j=1}^{R_B}$, per–ball facet counts $m_j=\Theta\!\big((r_j/\varepsilon_{\rm poly})^{(d-1)/2}\big)$, and slopes $\kappa,\lambda>0$ such that the two–layer sigmoidal MLP classifier built exactly as in \S\ref{sec:planar-general} produces a decision set $\widehat{C}$ with
\[
d_{\mathrm H}(\widehat{C},C)\ \le\ \varepsilon_{\rm cov}+\varepsilon_{\rm poly}+c_1\,\kappa^{-1}+c_2\,\lambda^{-1}\ \le\ \varepsilon,
\]
and, away from the inner/outer bands of thicknesses $\Theta(\kappa^{-1})$ and $\Theta(\lambda^{-1})$, agrees pointwise with $1_C$.

\end{theorem}
\begin{proof}
    This theorem is a trivial result of Heine–Borel (finite cover), ball$\to$polytope with $m_j$ as above, and Theorem~\ref{thm:union-hd}. Details are similar to Thm. \ref{thm:ball-compact-final}
\end{proof}
\paragraph{Complexity.}
Let $\varepsilon>0$ and split the budget
\[
\varepsilon_{\rm tot}
:=\varepsilon_{\rm cov}+\varepsilon_{\rm poly}+c_1\kappa^{-1}+c_2\lambda^{-1}\ \le\ \varepsilon,
\qquad
c_1=\log\!\tfrac{1-\eta}{\eta},\ \ c_2=\log\!\tfrac{1-\delta}{\delta}.
\]
For ball $j$ of radius $r_j$,
\[
m_j=\Theta\!\Big(\big(\tfrac{r_j}{\varepsilon_{\rm poly}}\big)^{\frac{d-1}{2}}\Big),\qquad
N_{\rm inner}=\sum_{j=1}^{R_B} m_j,\qquad
N_{\rm outer}=R_B,\qquad
\text{cost}=\Theta(N_{\rm inner}).
\]

\begin{remark}[Tropical viewpoint in higher $d$]
Let
\[
F(x)\;=\;\bigoplus_{\ell=1}^{m}\bigl(h_\ell \odot x^{-u_\ell}\bigr)
\;=\;\max_{1\le \ell\le m}\{\,h_\ell-\langle u_\ell,x\rangle\,\},\qquad
u_\ell\in\mathbb S^{d-1}.
\]
The hard classifier on a convex component is the tropical inequality
\[
\1\{F(x)\ge 0\}\;=\;\1\Big\{\,\bigoplus_{\ell}(h_\ell \odot x^{-u_\ell})\ \ge\ 0\,\Big\},
\]
i.e., membership in the intersection of supporting tropical half-spaces. Our sigmoidal MLP replaces each hard tropical term
\(h_\ell \odot x^{-u_\ell}\) by a steep gate
\(s_\ell(x)=\sigma\!\big(\kappa[h_\ell-\langle u_\ell,x\rangle]\big)\)
and aggregates by (soft) counting satisfied constraints in place of the hard tropical sum \(\oplus\).
The tropical hypersurface \(\Trop(F)\) is dual to the induced subdivision of the Newton polytope \(\delta(F)\); this duality (orthogonality, adjacency, and dimension reversal) carries over verbatim in \(\R^d\).
\end{remark}

\section{Numerical Studies}\label{sec:apps}

We do the numercial studies on the proposed geometry–aware construction with three different planar regions: (i) a \emph{single disk}, (ii) a \emph{two–disk union}, and (iii) a \emph{nonconvex swiss–roll} set approximated by a union of $\varepsilon$–balls. All models are purely sigmoidal MLP form as in \eqref{eq:finite-sum} trained with \emph{binary cross–entropy} (BCE) using the \emph{Adam} optimizer \cite{Kingma2015Adam}. We report the \emph{Brier score} \citep{Brier1950}, the area under curve (\emph{AUC}; \cite{Hanley1982}), and the \emph{intersection–over–union} at threshold $0.5$ (\emph{IoU};\cite{Jaccard1901}).

\paragraph{Experiment Setup.}
For the \textbf{single} and \textbf{two–disk} tasks we draw inputs uniformly from the window
\(K=[-2,2]^2\subset\mathbb{R}^2\); the training and evaluation set sizes are
\(\texttt{TRAIN\_N}=12{,}000\) and \(\texttt{TEST\_N}=3{,}000\).
Both disks have radius \(R=0.8\) with centers \(C_1=(-0.6,0.0)\) and \(C_2=(0.6,0.0)\).
We test hidden sizes \(H\in\{16,32\}\) and compare the five initializations.
Training uses \(80\) epochs, batch size \(512\), Adam with learning rate \(3\times 10^{-3}\).

In the geometry-aware construction (Secs.~\ref{sec:planar-convex}--\ref{sec:planar-general}),
each inner half-space gate takes the form
\[
s_\ell(x)\;=\;\sigma\!\bigl(\,\kappa\,[\,h_\ell-\langle u_\ell,x\rangle\,]\,\bigr),
\]
with logistic sigmoid \(\sigma(\cdot)\) and sharpness \(\kappa=\kappa_{\text{hidden}}=30.0\).
Evaluation uses a \(200\times 200\) grid over \(K\) and a threshold \(\tau = 0.5\).

\paragraph{Initialization methods.}
We compare our construction with four classic initialization schemes; throughout, training follows the exact protocol above, and no data–dependent pretraining is used.
\begin{enumerate}[leftmargin=1.2em]
\item \textbf{Random} — 
Draw weights from \emph{uniform} distributions. \citep{He2015Delving,Glorot2010Understanding}

\item \textbf{Xavier} — 
Initialize fully connected layers with Xavier–uniform to stabilize signal and gradients in sigmoidal multilayer perceptrons. \citep{Glorot2010Understanding}

\item \textbf{Kaiming (uniform)} — 
Use Kaiming–uniform for hidden layers and linear-consistent scaling for heads; serves as a rectifier-scaled baseline in our smooth setting. \citep{He2015Delving}

\item \textbf{He (normal)} — 
Sample hidden weights with He–\emph{normal} (Gaussian)  and keep linear-consistent heads. \citep{He2015Delving,Mishkin2016AllYouNeed}


\end{enumerate}

\subsection{Non-convex example : Swiss–roll}\label{sec:apps-fps}
Following Sec.~\ref{sec:planar-general}, we approximate the nonconvex swiss–roll target by a finite union of $\varepsilon$–balls on the plane, where $\varepsilon = 1.5$. Positives are identified on a uniform window and treated as candidate centers. We then (i) \emph{Voxel downsampling.} Let $v:=\varepsilon\times\texttt{VOXEL\_RATIO}=1.5\times0.6$. We quantize each candidate point $p\in\mathbb{R}^2$ to a grid cell by $q=\big\lfloor p/v \big\rfloor$ and retain a single representative per occupied cell (e.g., the first point). This “one-per-voxel’’ rule removes near-duplicate centers within distance $\lesssim v$, yielding a thinner, spatially uniform candidate set for next step.
(ii) apply \emph{Farthest Point Sampling} (FPS)—the classical greedy $k$–center algorithm \citep{Gonzalez1985}—to select at most $\texttt{FPS\_BUDGET}=120$ centers. The resulting cover drives the constructive initializer: for each selected center we place $H_{\text{sides}}=24$ steep half–planes (inner sharpness $\kappa_{\text{hidden}}=8.0$, per–disk sharpness $\kappa_{\text{disk}}=6.0$), and use an OR–like head $\alpha\!\left(\sum_k s_k-\tau\right)$ with $\alpha=6.0$ and $\tau=0.5$. This is the example construction of the two–layer sigmoidal construction from Sec.~\ref{sec:planar-general} within the form \eqref{eq:finite-sum}.

\begin{figure}[H]
  \centering
  \includegraphics[width=0.82\linewidth]{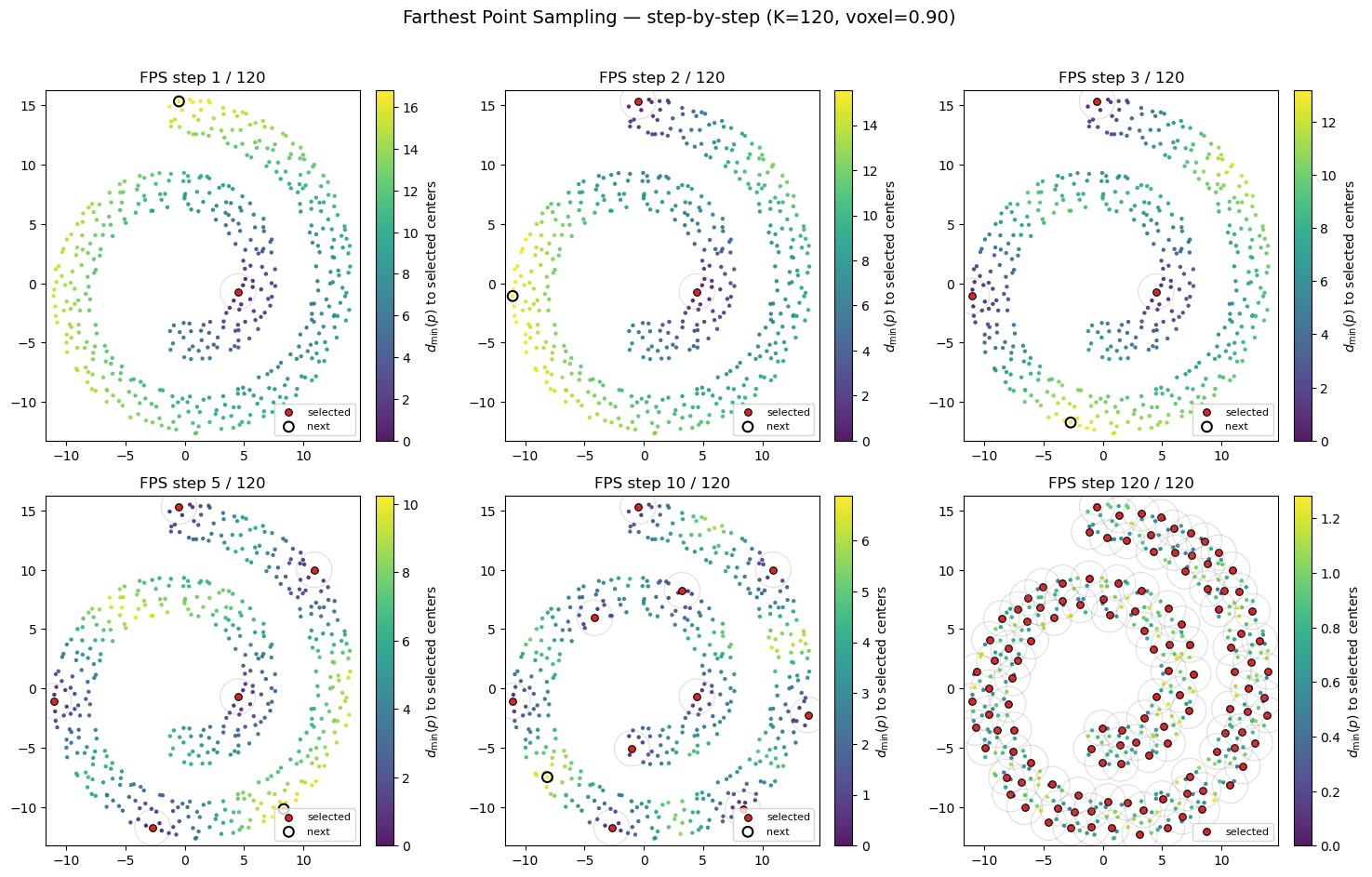}
  \caption{\textbf{How Farthest Point Sampling choose centers}}
  \label{fig:fps-pipeline}
\end{figure}

\paragraph{Metrics.}
Let $\{(x_i,y_i)\}_{i=1}^{N}$ be a binary dataset with $y_i\in\{0,1\}$.
The model outputs a logit $z_i\in\mathbb{R}$ and a probability
$p_i:=\sigma(z_i)\in[0,1]$, where $\sigma(t)=(1+e^{-t})^{-1}$.
Write $P:=\{i:\,y_i=1\}$ and $N:=\{i:\,y_i=0\}$.

\emph{(i) Brier score.}\cite{Brier1950} The mean squared error between calibrated probabilities and labels:
\[
\mathrm{Brier}\ :=\ \frac{1}{N}\sum_{i=1}^{N}\bigl(p_i-y_i\bigr)^2.
\]

\emph{(ii) Area under the ROC curve (AUC).}\cite{Hanley1982} Using scores $p_i$ (higher means more positive), the empirical AUC is the Wilcoxon–Mann–Whitney statistic:
\[
\mathrm{AUC}\ :=\ \frac{1}{|P|\,|N|}\sum_{i\in P}\sum_{j\in N}
\Bigl[\ \1\{p_i>p_j\}\;+\;\tfrac{1}{2}\,\1\{p_i=p_j\}\ \Bigr],
\]

\emph{(iii) Intersection-over-Union at threshold $\tau$ (IoU@$\tau$).}\cite{Jaccard1901}
Let $\widehat{y}_i(\tau):=\1\{p_i\ge\tau\}$ be the predicted label at threshold $\tau$ (we use $\tau=0.5$ in this section). Then
\[
\mathrm{IoU}(\tau)\ :=\
\frac{\sum_{i=1}^{N}\1\{\widehat{y}_i(\tau)=1,\ y_i=1\}}
{\sum_{i=1}^{N}\1\{\widehat{y}_i(\tau)=1\ \text{or}\ y_i=1\}}
\]

\subsection{Quantitative results}\label{sec:apps-quant}
Table~\ref{tab:single-double} summarizes single– and two–disk results under the exact protocol above. The geometry–aware initializer provides a near–perfect boundary at \emph{initialization} for the single–disk case (AUC $=1.0000$, IoU $\ge 0.9814$) and markedly stronger IoU for the two–disk union (AUC $\ge 0.9890$, IoU $\ge 0.7395$) compared to geometry–agnostic baselines. After training, all methods reach near–perfect AUC and high IoU; nonetheless, our construction is the only one that starts with a closed boundary on the tasks, reducing the reliance on optimization to \emph{refine} rather than \emph{discover} geometry.

\begin{table}[H]
\centering

\scriptsize
\setlength{\tabcolsep}{5pt}
\begin{tabular}{lllrrrrrr}
\toprule
Case & $H$ & Init & Init Brier & Init AUC & Init IoU & Final Brier & Final AUC & Final IoU\\
\midrule
\multirow{10}{*}{Single}
&16&random & 0.4176 & 0.4861 & 0.1255 & 0.0089 & 0.9999 & 0.9548\\
&16&xavier & 0.2086 & 0.4965 & 0.0805 & 0.0056 & 1.0000 & 0.9767\\
&16&kaiming& 0.3943 & 0.5154 & 0.1336 & 0.0073 & 1.0000 & 0.9727\\
&16&he     & 0.1441 & 0.5194 & 0.0000 & 0.0045 & 1.0000 & 0.9794\\
&16&\textbf{ours}   & \textbf{0.0216} & \textbf{1.0000} & \textbf{0.9954} & 0.0066 & 0.9999 & 0.9634\\
&32&random & 0.1263 & 0.5119 & 0.0000 & 0.0034 & 1.0000 & 0.9828\\
&32&xavier & 0.3518 & 0.5197 & 0.1295 & 0.0049 & 1.0000 & 0.9760\\
&32&kaiming& 0.2322 & 0.5034 & 0.1050 & 0.0038 & 1.0000 & 0.9874\\
&32&he     & 0.1875 & 0.5116 & 0.0000 & 0.0038 & 1.0000 & 0.9807\\
&32&\textbf{ours}   & \textbf{0.0193} & \textbf{1.0000} & \textbf{0.9814} & 0.0026 & 1.0000 & 0.9874\\
\midrule
\multirow{10}{*}{Double}
&16&random & 0.2365 & 0.5038 & 0.1626 & 0.0057 & 0.9996 & 0.9709\\
&16&xavier & 0.3645 & 0.5047 & 0.2297 & 0.0081 & 0.9995 & 0.9573\\
&16&kaiming& 0.4116 & 0.4951 & 0.1771 & 0.0070 & 0.9996 & 0.9611\\
&16&he     & 0.2374 & 0.5098 & 0.1626 & 0.0045 & 0.9998 & 0.9795\\
&16&\textbf{ours}   & \textbf{0.0659} & \textbf{0.9890} & \textbf{0.7395} & 0.0102 & 0.9993 & 0.9496\\
&32&random & 0.2538 & 0.5003 & 0.1972 & 0.0033 & 1.0000 & 0.9870\\
&32&xavier & 0.3771 & 0.5001 & 0.1867 & 0.0055 & 0.9998 & 0.9719\\
&32&kaiming& 0.2227 & 0.4902 & 0.0000 & 0.0031 & 0.9999 & 0.9870\\
&32&he     & 0.2235 & 0.4963 & 0.0806 & 0.0043 & 0.9999 & 0.9774\\
&32&\textbf{ours}   & \textbf{0.0601} & \textbf{0.9934} & \textbf{0.7757} & 0.0066 & 0.9999 & 0.9751\\
\bottomrule

\end{tabular}
\caption{Results of single and two–disk tasks: Brier, AUC, and IoU@0.5 at initialization and after training (single/two–disk: epochs $=80$, batch $=512$, learning rate $=3\!\times\!10^{-3}$). Best initialization per block in \textbf{bold}.}
\label{tab:single-double}
\end{table}

On the \textbf{single–disk} task, our initializer achieves near–perfect geometry at initialization (AUC $=1.0000$, IoU $>0.98$). After training, all methods converge to near–perfect AUCs. On the \textbf{two–disk union}, our OR-like aggregator produces a faithful union boundary already at initialization, reflected in high IoU; subsequent training mainly refines calibration.

\subsection{Qualitative results}\label{sec:apps-qual}
We visualize (i) the data distributions (positives / negatives), (ii) decision maps at initialization and after training for $H=32$, and (iii) loss curves for each $H$.
\begin{figure}[H]
\centering
\begin{subfigure}[t]{0.31\textwidth}
  \centering\includegraphics[width=\linewidth]{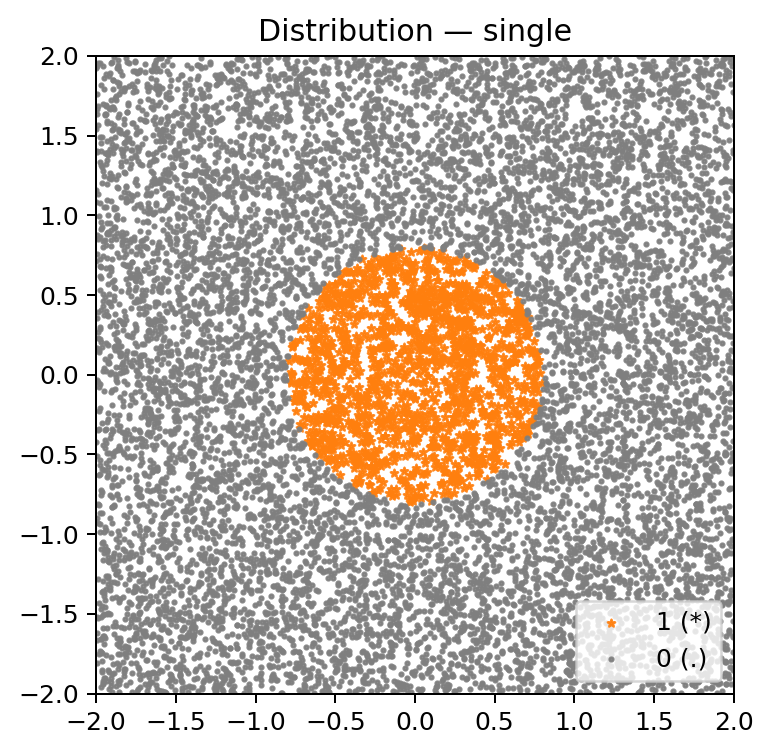}
  \caption{Single—data distribution}
\end{subfigure}\hfill
\begin{subfigure}[t]{0.31\textwidth}
  \centering\includegraphics[width=\linewidth]{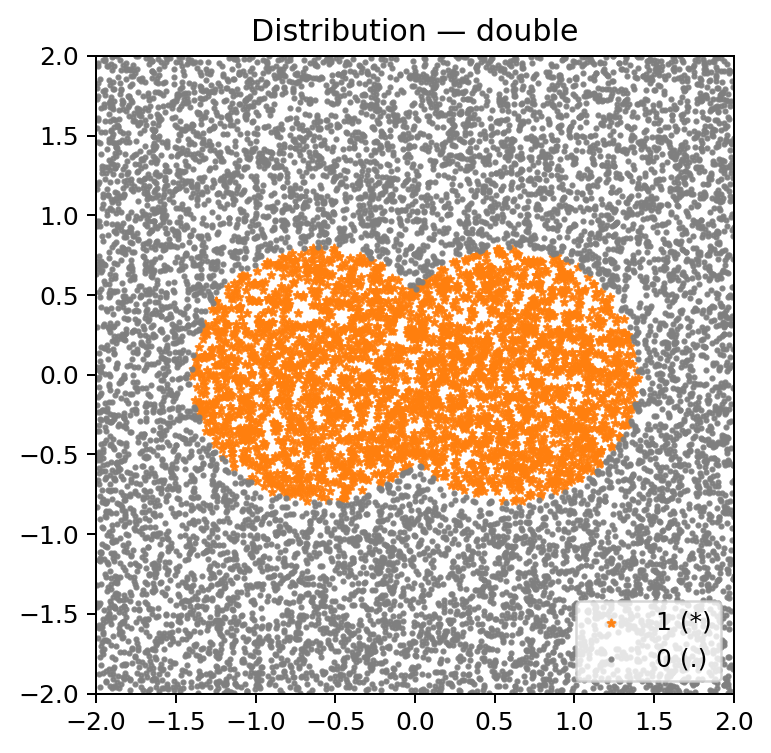}
  \caption{Double—data distribution}
\end{subfigure}\hfill
\begin{subfigure}[t]{0.27\textwidth}
  \centering\includegraphics[width=\linewidth]{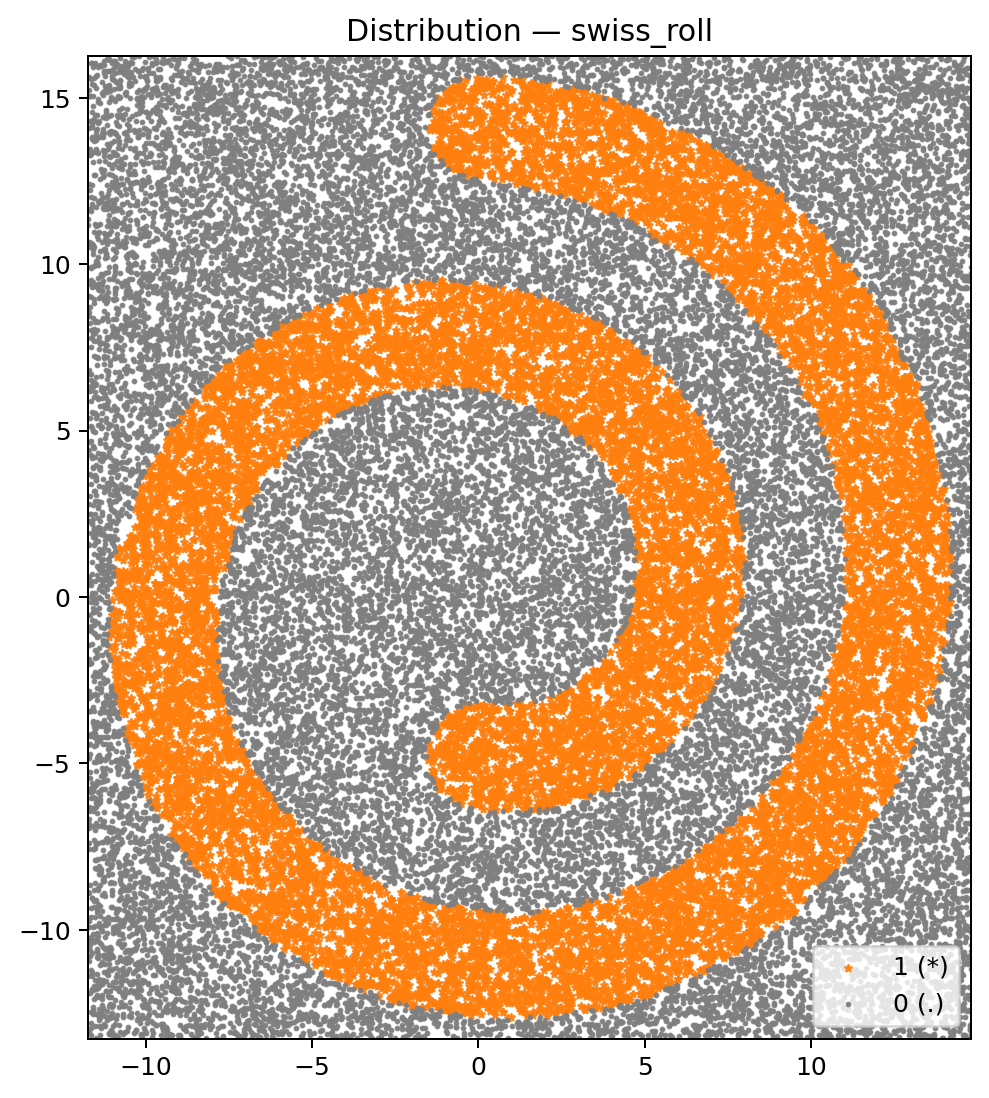}
  \caption{Swiss–roll—uniform labels}
\end{subfigure}
\caption{Distributions used in the three cases, positives are colored in orange, while negatives are in gray.}
\end{figure}

\begin{figure}[H]
\centering
\begin{subfigure}[t]{0.8\textwidth}
  \centering\includegraphics[width=\linewidth]{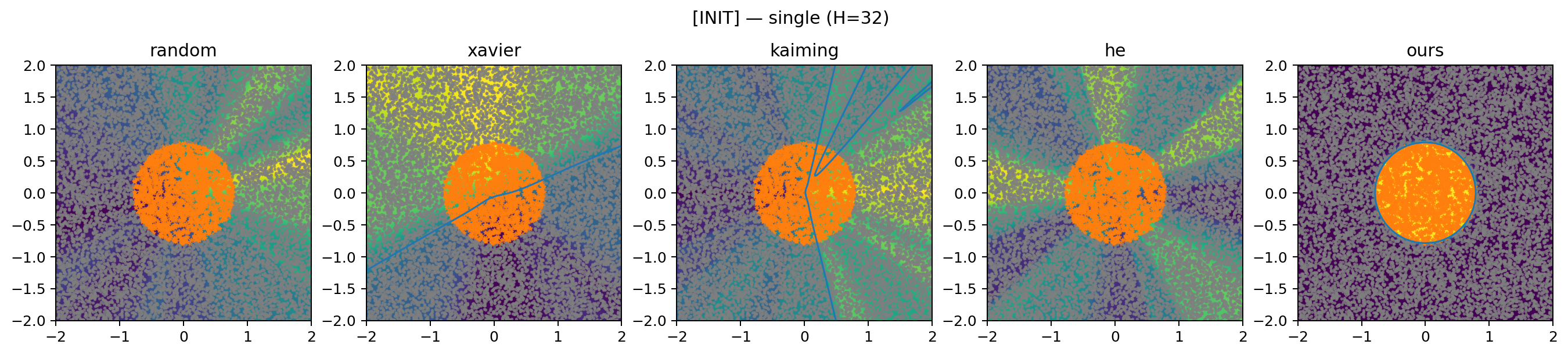}
  \caption{Single, $H{=}32$—initialization}
\end{subfigure}\hfill
\begin{subfigure}[t]{0.8\textwidth}
  \centering\includegraphics[width=\linewidth]{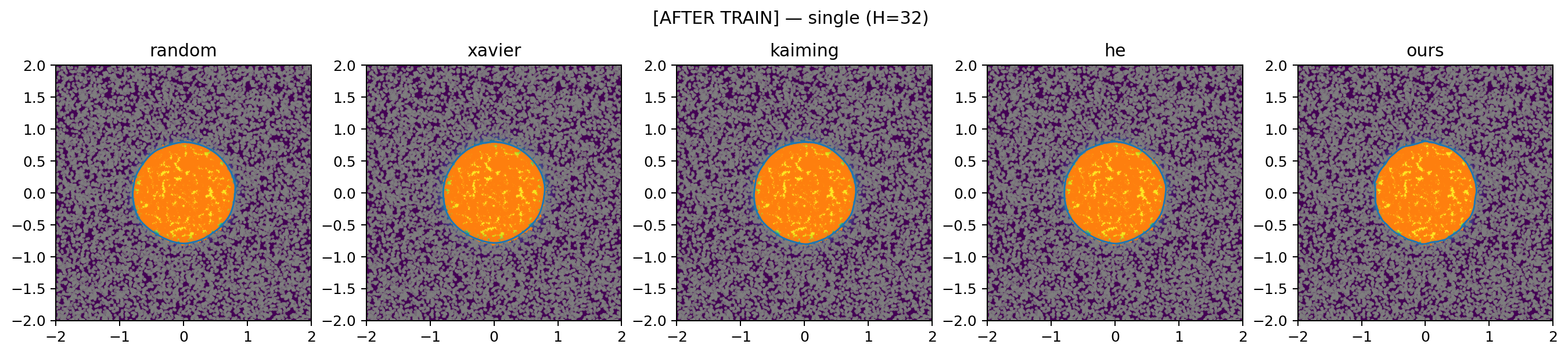}
  \caption{Single, $H{=}32$—after training}
\end{subfigure}

\vspace{0.4em}
\begin{minipage}{0.8\textwidth}
  \centering
  \includegraphics[width=\linewidth]{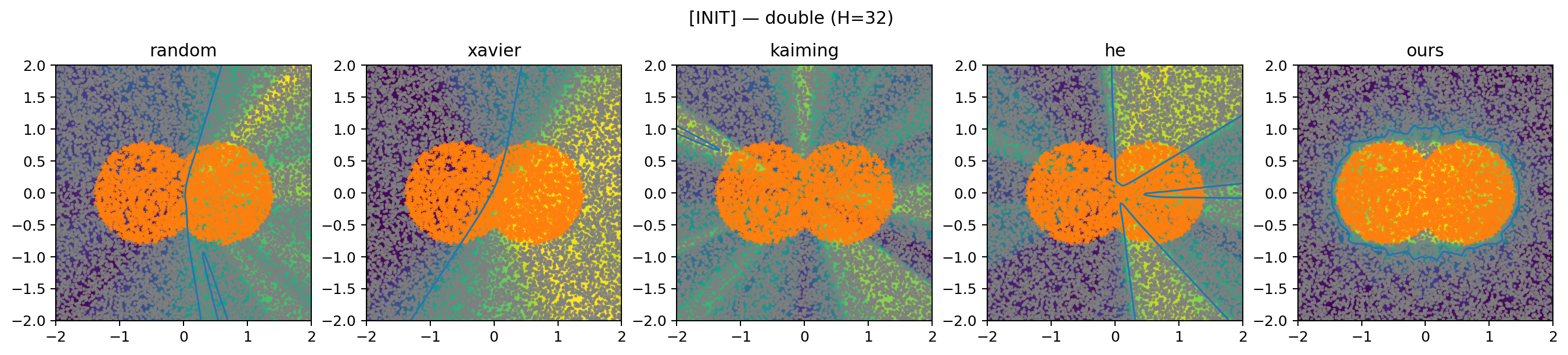}\\
  \small (Double, $H{=}32$—initialization)
\end{minipage}\hfill
\begin{minipage}{0.8\textwidth}
  \centering
  \includegraphics[width=\linewidth]{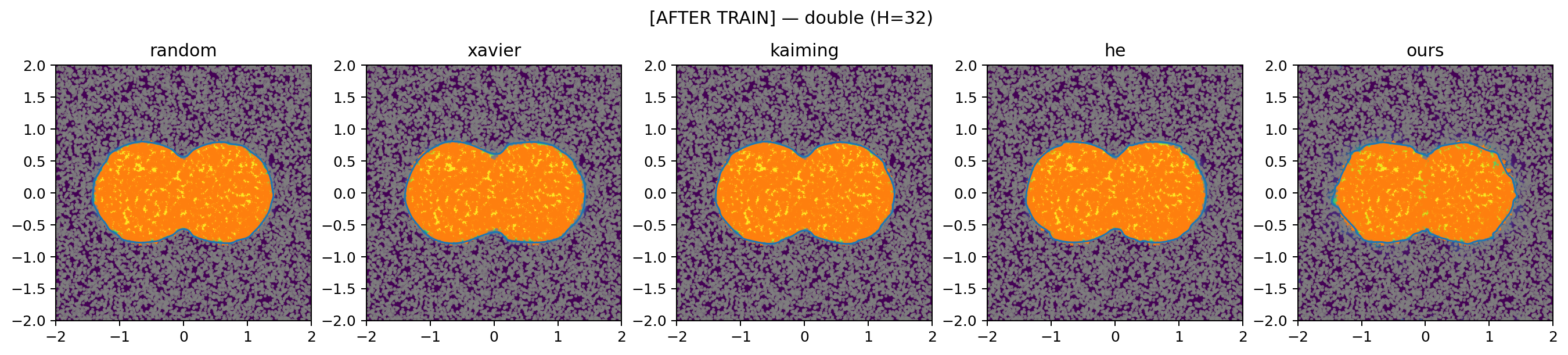}\\
  \small (Double, $H{=}32$—after training)
\end{minipage}
\caption{Decision maps for the five initializations. The blue contour line marks the $0.5$ is the decision boundary ($\tau=0.5$). The background shows the predicted probability $\sigma(z)$; warmer colors indicate higher probability (closer to $1$).}
\end{figure}

\begin{figure}[H]
\centering
\begin{subfigure}[t]{0.48\textwidth}
  \centering\includegraphics[width=\linewidth]{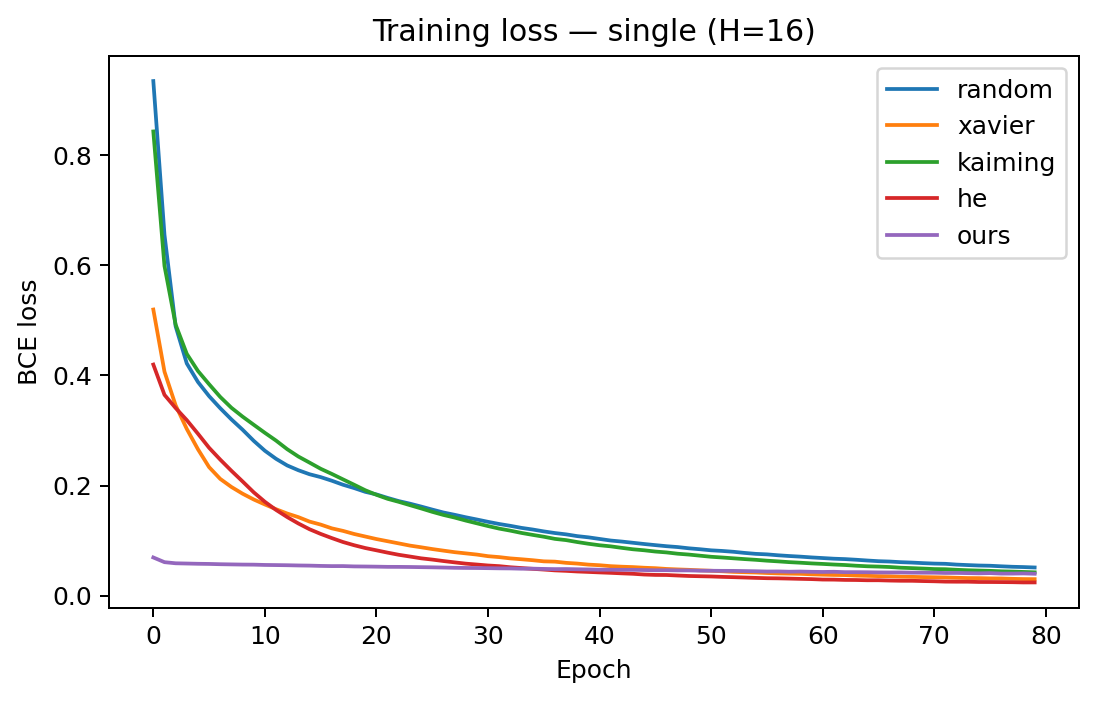}
  \caption{Single—loss ($H{=}16$)}
\end{subfigure}\hfill
\begin{subfigure}[t]{0.48\textwidth}
  \centering\includegraphics[width=\linewidth]{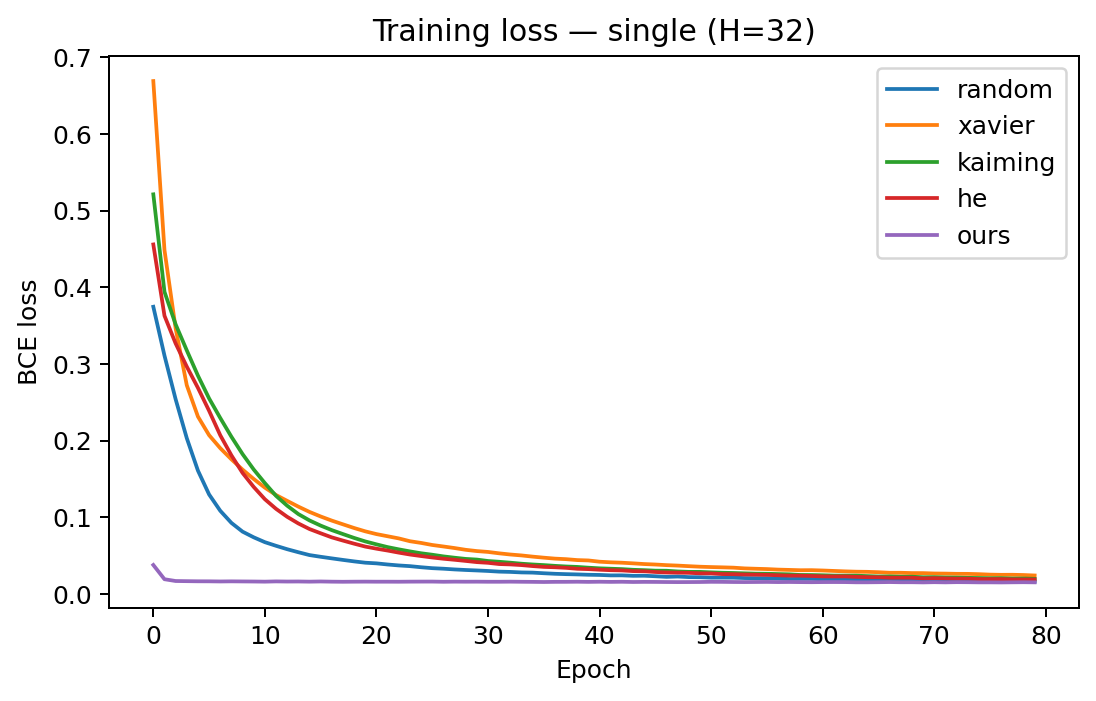}
  \caption{Single—loss ($H{=}32$)}
\end{subfigure}

\vspace{0.4em}
\begin{subfigure}[t]{0.48\textwidth}
  \centering\includegraphics[width=\linewidth]{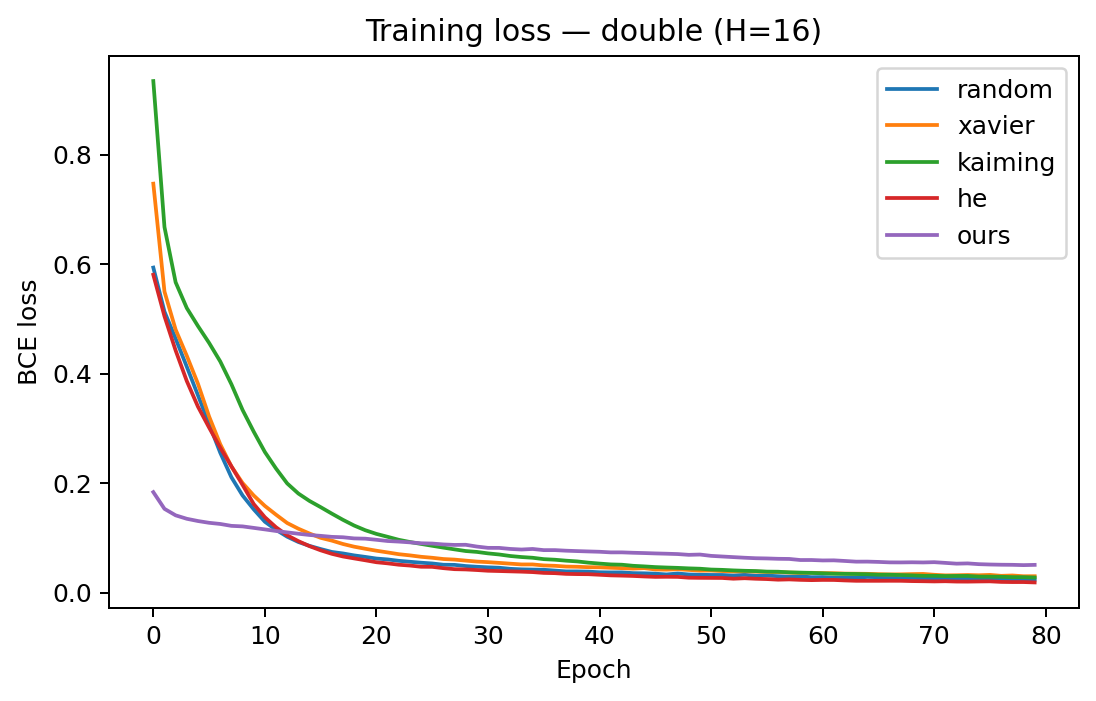}
  \caption{Double—loss ($H{=}16$)}
\end{subfigure}\hfill
\begin{subfigure}[t]{0.48\textwidth}
  \centering\includegraphics[width=\linewidth]{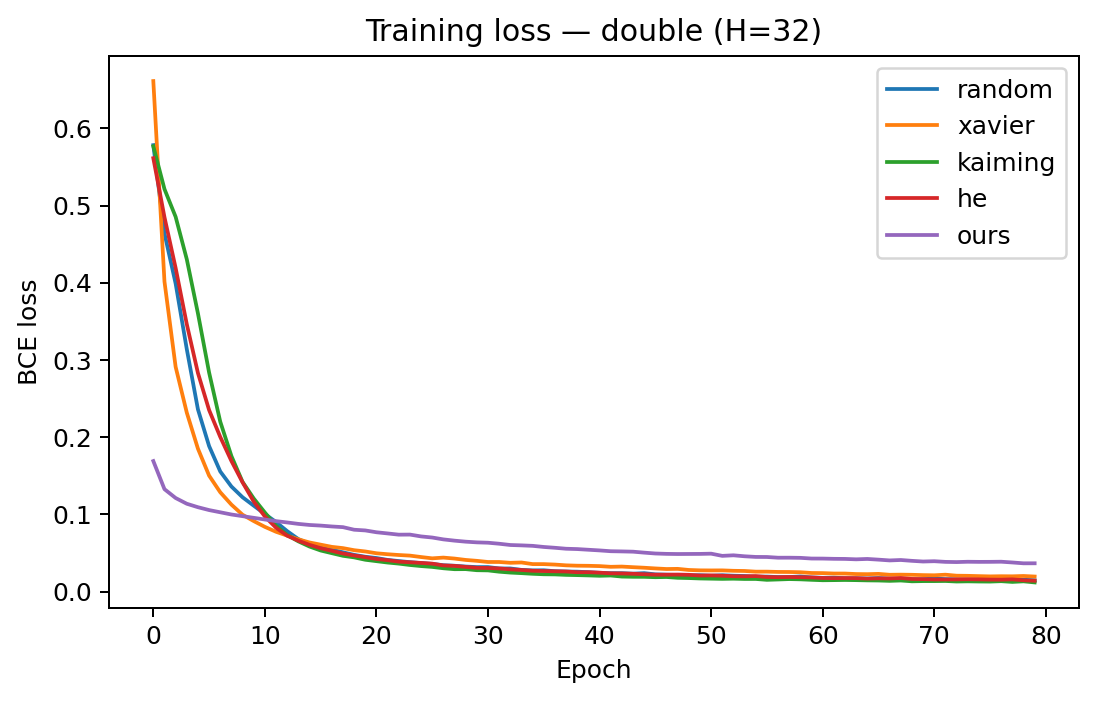}
  \caption{Double—loss ($H{=}32$)}
\end{subfigure}
\caption{Training curves across initializations.}
\end{figure}

\subsection{Swiss–roll}\label{sec:apps-swiss}
The network starts from a geometrically consistent initialization (INIT: $\mathrm{BCE}=0.5823$, Brier$=0.1615$, IoU$=0.6477$, AUC$=0.9336$) and is then refined by training with early stopping, after which performance improves to (AFTER(early-stopped at $epoch =67$: $\mathrm{BCE}=0.0454$, Brier$=0.0060$, IoU$=0.9845$, AUC$=0.9999$). Fig.~\ref{fig:swiss-roll-ours} shows the initialized and final probability maps, together with the loss curve.

\begin{figure}[H]
\centering
\begin{subfigure}[t]{0.32\textwidth}
  \centering\includegraphics[width=\linewidth]{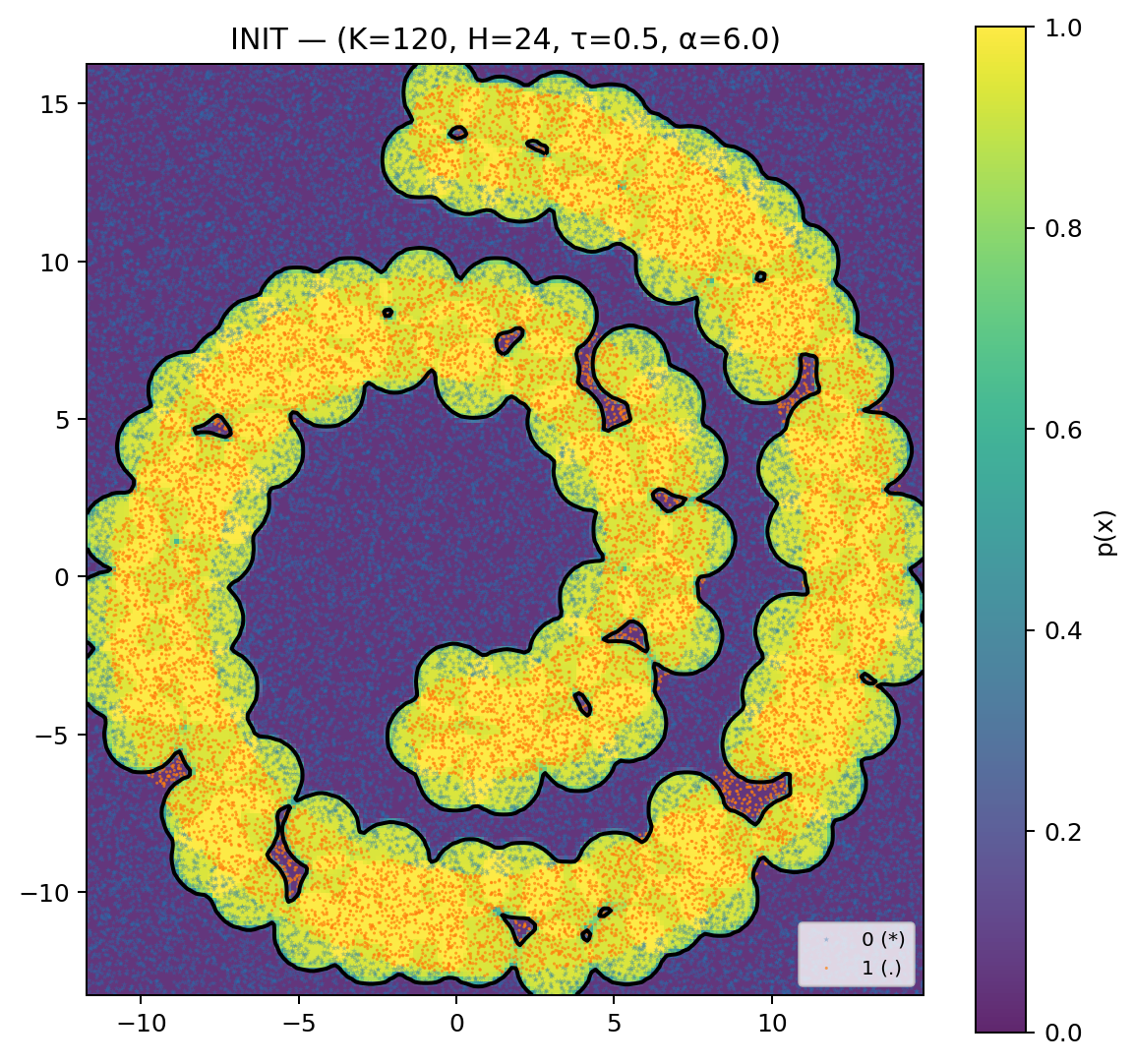}
  \caption{Initialization}
\end{subfigure}\hfill
\begin{subfigure}[t]{0.32\textwidth}
  \centering\includegraphics[width=\linewidth]{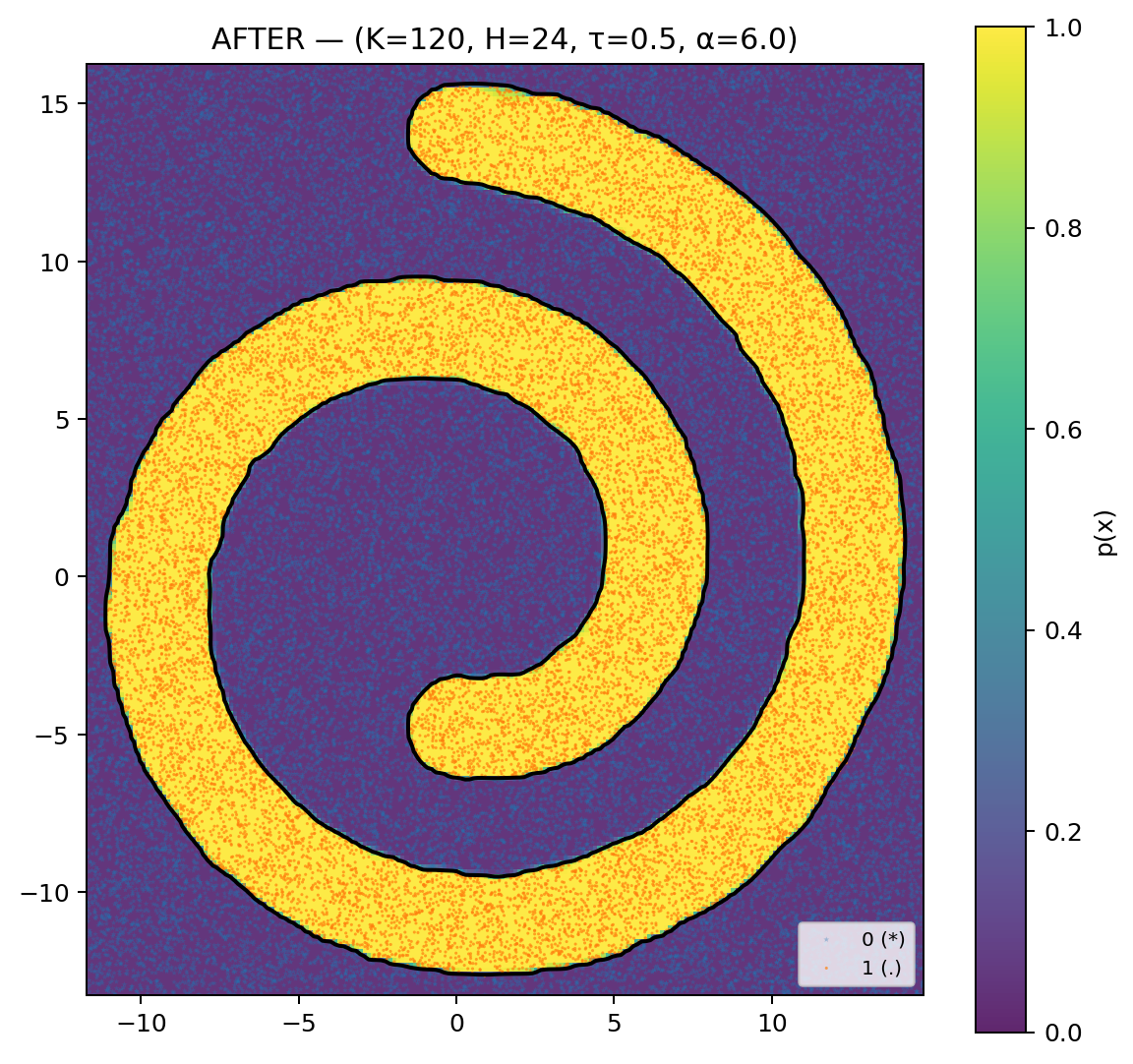}
  \caption{After training}
\end{subfigure}\hfill
\begin{subfigure}[t]{0.32\textwidth}
  \centering\includegraphics[width=\linewidth]{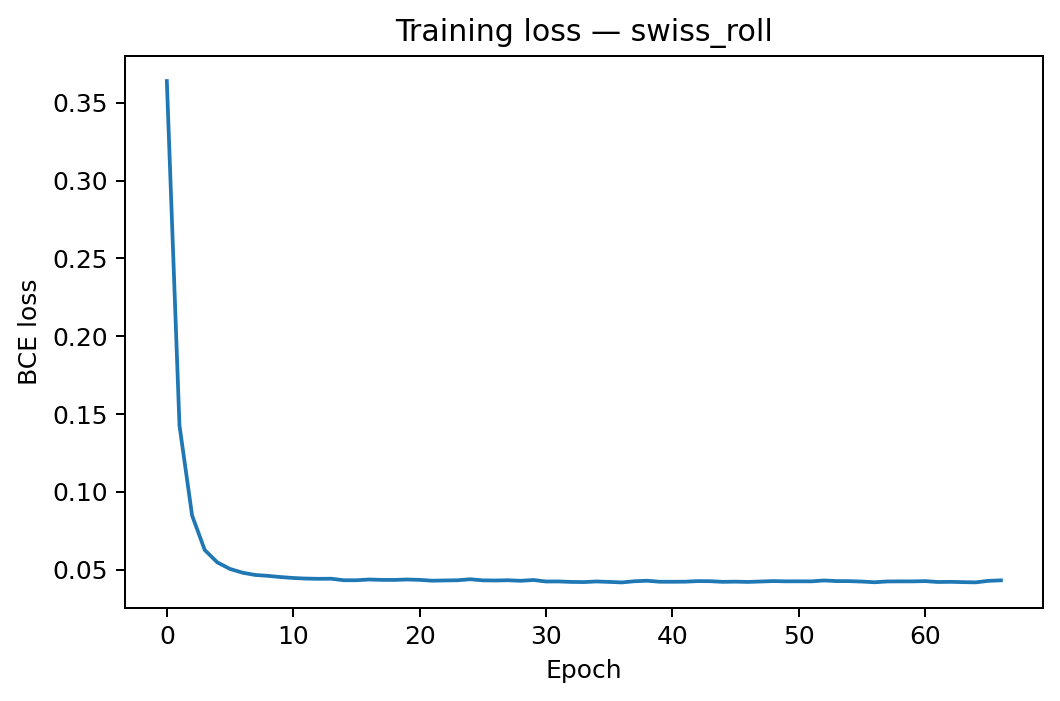}
  \caption{Swiss–roll training loss}
\end{subfigure}
\caption{Swiss–roll case: geometric construction and post–training refinement.}
\label{fig:swiss-roll-ours}
\end{figure}

Although our construction may seems to be artificial, the construction plays an important role in describing the inductive bias of purely sigmoidal MLPs. By compiling target geometry into sigmoidal half–space gates (Secs.~\ref{sec:planar-convex}–\ref{sec:planar-general}), we obtain networks whose \emph{initialized} decision boundaries already align with the task geometry up to a narrow band. On planar numerical experiments, we observed near–perfect AUC and high IoU at initialization for the constructive models, while standard random/Xavier/Kaiming/He initializations start near chance and rely entirely on optimization and training to discover the geometric boundary.

\section{Conclusions}\label{sec:conclusion}

We revisited Universal Approximation Theorem for sigmoidal Multi-Layer Perceptrons through a constructive, tropical geometry–aware sense. Within the finite–sum form \eqref{eq:finite-sum}, our method build a covering of the target region into sigmoidal half–space gates, yielding decision boundaries that agree with the prescribed geometry at \emph{initialization}; subsequent optimization is not essential and primarily refines calibration rather than discovering the boundary. This positions our approach as a smooth counterpart to tropical, boundary-first reasoning while remaining entirely sigmoidal.

\paragraph{Limitations.}
Our analysis and constructions focus on planar (\(d{=}2\)) binary classification with logistic sigmoid function; while Sec.~7 gives the outlines to higher-dimensional extensions, we have not yet provided the same level of quantitative study beyond the plane. The construction algorithm depends on a reasonable covering (e.g., by disks); extremely fine geometric detail or high curvature can increase the required size of the network. Finally, although training typically preserves the intended union semantics, incorrect centering/scaling of the head can degrade stability.

\paragraph{Future Works.}
Two directions appear especially promising: (i) principled cover selection with approximation guarantees and budget-aware trade-offs; and (ii) integrating tropical certificates with sigmoidal training, e.g., enforcing margin/coverage constraints derived from support functions during optimization. More broadly, we view geometry-aware compilation as a practical bridge between classical UAT and tropical structure: it turns \emph{where} the boundary should be (Sections~5–6) into \emph{how} to place weights so that the MLP starts, which allow learning focus on calibration and robustness rather than boundary discovery.


\bibliographystyle{plainnat}
\bibliography{ref}

\end{document}